%% file: main.tex
\documentclass[nohyperref]{article}

\input{preamble}
\def\phiin{\phi_{\tn{in}}}
\def\bmphiin{\bmphi_{\tn{in}}}
\def\phiout{\phi_{\tn{out}}}
\def\bmphiout{\bmphi_{\tn{out}}}

\definecolor{mygray}{RGB}{220,220,220}

\usepackage{hyperref}

\newif\ifarXiv
\arXivtrue
\usepackage[accepted]{icml2022}
\makeatletter\ifarXiv
\def\ICML@appearing{\vspace*{-22.5pt}}
\fancypagestyle{myplain}
{\fancyhf{}
  \renewcommand\headrulewidth{0pt}
  \renewcommand\footrulewidth{0pt}
  \fancyfoot[C]{\raisebox{-31pt}{\thepage}}
}
\fi\makeatother

\usepackage{amsmath}
\usepackage{amssymb}
\usepackage{mathtools}
\usepackage{amsthm}

\usepackage[capitalize,noabbrev]{cleveref}

\theoremstyle{plain}
\newtheorem{theorem}{Theorem}[section]
\newtheorem{proposition}[theorem]{Proposition}
\newtheorem{lemma}[theorem]{Lemma}

\theoremstyle{definition}

\theoremstyle{remark}

\usepackage[textsize=tiny]{todonotes}

\icmltitlerunning{Network Approximation in Terms of Intrinsic Parameters}

\begin{document}

\twocolumn[
\icmltitle{Deep Network Approximation in Terms of Intrinsic Parameters}



\icmlsetsymbol{equal}{*}

\begin{icmlauthorlist}
\icmlauthor{Zuowei Shen}{nus}
\icmlauthor{Haizhao Yang}{umd}
\icmlauthor{Shijun Zhang}{equal,nus}
\end{icmlauthorlist}

\icmlaffiliation{nus}{Department of Mathematics, National University of Singapore, Singapore\ifarXiv.\fi}
\icmlaffiliation{umd}{Department of Mathematics, University of Maryland, College Park, United States}

\icmlcorrespondingauthor{Shijun Zhang}{zhangshijun@u.nus.edu}

\icmlkeywords{Intrinsic Parameters,   Neural Network Approximation,  Exponential Convergence,  Transfer Learning}

\vskip 0.3in
]



\printAffiliationsAndNotice{\ifarXiv\textsuperscript{*}Correspondence to: Shijun Zhang (zhangshijun@u.nus.edu).\fi}  
\ifarXiv\thispagestyle{myplain}\fi

\begin{abstract}
One of the arguments to explain the success of deep learning is the powerful approximation capacity of deep neural networks. Such capacity is generally accompanied by the explosive growth of the number of parameters, which, in turn, leads to high computational costs.  It is of great interest to ask whether we can achieve successful deep learning with a small number of learnable parameters adapting to the target function. From an approximation perspective, this paper shows that the number of parameters that need to be learned can be significantly smaller than people typically expect. 
First, we theoretically design ReLU networks with a few learnable parameters to achieve an attractive approximation. We prove by construction that, for any Lipschitz continuous function $f$ on $[0,1]^d$ with a Lipschitz constant $\lambda>0$, a ReLU network with $n+2$ intrinsic parameters (those depending on $f$) can approximate $f$ with an  exponentially small error $5\lambda \sqrt{d}\,2^{-n}$. Such a result is generalized to generic continuous functions. Furthermore, we show that the idea of learning a small number of parameters to achieve a good approximation can be numerically observed. We conduct several experiments to verify that training a small part of parameters can also achieve good results for classification problems if other parameters are pre-specified or pre-trained from a related problem. 
\end{abstract}

\section{Introduction}

Deep neural networks have
recently achieved great success in a large number of real-world applications. However, 
the success in deep neural networks is generally accompanied by the explosive growth of computation and parameters. 
This follows a natural problem: how to handle computationally expensive deep learning models with limited computing resources. 
This problem is challenging and has been widely studied. 
Numerous model compression and acceleration
methods have recently been proposed, e.g., parameter pruning and quantization \cite{efficient:nets:2015,7780890,pmlr-v37-gupta15}, low-rank factorization \cite{10.5555/2968826.2968968,BMVC.28.88}, transferred compact convolutional filters \cite{DBLP:journals/corr/HowardZCKWWAA17,8578572,9010015}, knowledge distillation \cite{2015arXiv150302531H,9008829,9381661}. See a survey of these methods in \cite{survey:compression:acceleration}.
This paper explores an approximation perspective for the problem mentioned above. We adopt the approximation perspective since the approximation power is a key ingredient for the performance of deep neural networks. In other words,
The goal of this paper is to investigate how to reduce the number of parameters that need to be learned from an approximation perspective.
We will provide both theoretical and numerical examples to show that adjusting only a small number of parameters is enough to achieve good results if the network architecture is properly designed.



To design a simple and computable hypothesis function $\phi$ to approximate a target function $f\in\scrF$ well via adjusting only a small number of $f$-dependent parameters, where $\scrF$ is a given target function space, 
we have the following two main ideas, motivated by the power of deep neural networks via function composition:
 \begin{itemize}
 	 \item The main component of $\phi$ is determined by $\scrF$ and can be shared as fixed parameters for all functions in $\scrF$. These shared parameters can be determined a priori or learned from any function in $\scrF$.
 	\item  $\phi$ is constructed via the composition of a few functions, only two of which is determined by $f$ with a small number of $f$-dependent parameters.
 \end{itemize}
In particular, we have the following construction
\begin{equation}\label{eq:basic:arc}
	\phi = \phi_{f,\calR}\circ\bmphiout\circ\bmphi_f\circ\bmphiin,
\end{equation}
where $\bmphiin$ and $\bmphiout$ are $f$-independent functions designed based on prior knowledge. We call $\bmphiin$ and $\bmphiout$ \textbf{inner-function} and \textbf{outer-function}, respectively. $\bmphi_f$ and $\phi_{f,\calR}$ are $f$-dependent functions.  
$\bmphi_f$ is the core part of the whole architecture and 
$\phi_{f,\calR}$ is a simple function for the purpose of adjusting the output range. 
If we use neural networks to implement the architecture in \eqref{eq:basic:arc}, most parameters are $f$-independent and stored in $\bmphiin$ and $\bmphiout$. We will show that good theoretical and numerical approximations can be achieved by adjusting only a small number of parameters in $\bmphi_f$ and $\phi_{f,\calR}$.

We will focus on the rectified linear unit (ReLU) activation function and use it to demonstrate our ideas. It would be interesting for future work to extend our work to other activation functions. 
First, we use the architecture in \eqref{eq:basic:arc} to theoretically design ReLU networks to approximate (H\"older) continuous functions within an exponentially small approximation error in terms of the number of  $f$-dependent parameters. As we shall see later, in an extreme case, adjusting three intrinsic parameters is enough to achieve an arbitrarily small approximation error.
Next, 
we design a ReLU convolutional neural network (CNN), similar to the architecture in \eqref{eq:basic:arc}, 
to conduct several experiments to numerically verify that
training a small number of intrinsic parameters are enough to achieve good results.

In fact, the architecture \eqref{eq:basic:arc} has a more general form as follows:
\begin{equation}\label{eq:general:arc}
	\phi = \phi_{f,\calR}\circ\bmphi_k\circ\bmphi_{f,k-1}\circ\bmphi_{k-1}\circ\cdots\circ\bmphi_{f,1}\circ\bmphi_{1},
\end{equation}
where $\bmphi_1,\cdots,\bmphi_k$ are $f$-independent functions, which are designed based on prior knowledge. $\bmphi_{f,1},\cdots,\bmphi_{f,k-1}$ are the core parts of the whole architecture. $\phi_{f,\calR}$ is a simple function adjusting the output range. $\bmphi_{f,1},\cdots,\bmphi_{f,k-1}$ and $\phi_{f,\calR}$ are determined by $f$.
This paper only focuses on the form in \eqref{eq:basic:arc}. The study of the general form in \eqref{eq:general:arc} is left as future work.

Let us further discuss why we emphasize the parameters depending on the target function.
It was shown in \cite{yarotsky18a,shijun2,shijun:thesis,shijun6} that the approximation error $\calO(n^{-2/d})$ is (nearly) optimal for ReLU networks with $\calO(n)$ parameters to approximate Lipschitz continuous functions on $[0,1]^d$. To gain better approximation errors, existing results either consider smaller target function spaces  \cite{yarotsky:2019:06,yarotsky2017,shijun3,barron1993,Weinan2019,doi:10.1002/mma.5575,bandlimit} or introduce new activation functions  \cite{shijun4,shijun7,shijun5,pmlr-v139-yarotsky21a}. 
Observe that, in many existing results, most parameters of networks constructed to approximate the target function $f$ are independent of $f$.
We propose a new perspective to study the approximation error in terms of the number of parameters depending on $f$, which are called \textbf{intrinsic parameters}, excluding those independent of $f$. We prove by construction that the approximation error can be greatly improved from our new perspective.

Our main contributions can be summarized as follows.
\begin{itemize}
	\item First, we propose a compositional architecture in \eqref{eq:basic:arc} and use such an architecture to design networks to approximate target functions. In particular, we construct a ReLU network with $n+2$ intrinsic parameters to approximate a H\"older continuous function $f$ on $[0,1]^d$ with an error  $5\lambda d^{\alpha/2}2^{-\alpha n}$
	measured in the $L^p$-norm for $p\in [1,\infty)$, where $\alpha\in (0,1]$ and $\lambda>0$ are the H\"older order and constant, respectively. Such a result is generalized to generic continuous functions. See Theorem~\ref{thm:main} for more details.
	
	\item We generalize the result in the $L^p$-norm  for $p\in[1,\infty)$ to a new one measured in the $L^\infty$-norm. Such a generalization is at a price of more intrinsic parameters. Refer to Theorem~\ref{thm:mainInfty} for more details.
	
	\item We further extend our results and show that the number of intrinsic parameters can be reduced to three. To be precise, ReLU networks with three intrinsic parameters can achieve an arbitrarily small error for approximating H\"older continuous function on $[0,1]^d$. In this scenario, extremely high precision is required as we shall see later.
	
	\item Finally, we conduct several experiments to 
	numerically verify that training a small part of parameters can achieve good results for classification problems if other parameters are pre-trained from a part of samples. 
\end{itemize}

The rest of this paper is organized as follows.
We first present our main theorems and discuss related work in Section~\ref{sec:main}.
These theorems are proved in the appendix. Next, we conduct several experiments to numerically verify our theory in Section~\ref{sec:experiments}.
Finally, Section~\ref{sec:conclusion} concludes this paper with a short discussion.
\section{Main results and further interpretation}\label{sec:main}

In this section, we first present our main theorems and then discuss related work. The proofs of these theorems can be found in the appendix.

\subsection{Main results}
Denote $C([0,1]^d)$ as  the space of continuous functions defined on $[0,1]^d$. 
Let $\calH_W(d_1,d_2)$ denote the function space consisting of all functions realized by ReLU networks with $W$ parameters mapping from $\R^{d_1}$ to $\R^{d_2}$, i.e.,
\begin{equation*}
\begin{split}
		\calH_W(d_1,d_2)\coloneqq \Big\{g:\ &\tn{$g:\R^{d_1}\to\R^{d_2}$ is realized by a }\\
			& \tn{ReLU network with $W$ parameters}\Big\}.
\end{split}
\end{equation*}
Let $\calH(d_1,d_2)\coloneqq \bigcup_{W=1}^{\infty} \calH_W(d_1,d_2)$.

For any $f\in C([0,1]^d)$, our goal is to construct two $f$-independent functions $\phi_1\in\calH(d,1)$ and $\phi_2\in\calH(n,1)$, and use $s\cdot(\phi_2\circ\phi_f \circ \phi_1)+b$ to approximate $f$, where $s\in [0,\infty)$, $b\in\R$, and $\phi_f\in \calH_n(1,n)$ are learned from $f$.
Under these settings, an approximation error $\omega_f(\sqrt{d}\, 2^{-n})+2^{-n+2}\omega_f(\sqrt{d})$ is attained as shown in the theorem below,
where the modulus of continuity of a continuous function $f\in C([0,1]^d)$ is defined as 
\begin{equation*}
	\omega_f(r)\coloneqq \sup\big\{|f(\bmx)-f(\bmy)|: \|\bmx-\bmy\|_2\le r,\ \bmx,\bmy\in [0,1]^d\big\}
\end{equation*}
for any $r\ge0$.
\begin{theorem}
	\label{thm:main}
	Given any $n\in\N^+$ and $p\in [1,\infty)$, there exist $\phi_1\in \calH_{2^{dn+4}}(d,1)$ and $\phi_2\in \calH_{2^{dn+5}n}(n,1)$ such that: For any $f\in C([0,1]^d)$, there exists a linear map $\calL:\R\to \R^n$ satisfying
	\begin{equation*}
	\begin{split}
		&\quad\ \big\|s\cdot(\phi_2\circ \calL \circ\phi_1)+b -f\big\|_{L^p([0,1]^d)}\\
		&\le \omega_f(\sqrt{d}\, 2^{-n})+2^{-n+2}\omega_f(\sqrt{d}),
	\end{split}
	\end{equation*}
	where $s=2\omega_f(\sqrt{d})$, $b=f(\bmzero)-\omega_f(\sqrt{d})$, and $\calL$ is a linear map given by $\calL(t)=(a_1t, a_2 t,\cdots,a_n t)$ with $a_1,a_2,\cdots,a_n\in [0,\tfrac{1}{3})$ determined by $f$ and $n$.
\end{theorem}


In Theorem~\ref{thm:main}, $s$ is a scale factor, $b$ is the bias for a vertical shift, and $a_1,a_2,\cdots,a_n\in [0,\tfrac{1}{3})$ are the key intrinsic parameters storing most of information of $f$. Clearly, $s\cdot (\phi_2\circ \calL \circ\phi_1)+b$ can be implemented by a ReLU network with $n+2$ intrinsic parameters.
$\phi_1$ and $\phi_2$ are independent of the target function $f$ and can be implemented by ReLU networks. Remark that 
the architecture $s\cdot(\phi_2\circ \calL \circ\phi_1)+b$ in Theorem~\ref{thm:main} can be rewritten as $\phi_{f,\calR}\circ\phiout\circ\bmphi_f\circ\phiin$, where $\phiin=\phi_1$,
$\bmphi_f=\calL$, $\phiout=\phi_2$, and $\phi_{f,\calR}$ is a linear function given by $\phi_{f,\calR}(x)=sx+b$. 
Clearly, it is a special case of the architecture in \eqref{eq:basic:arc}.

Note that 
the approximation error in Theorem~\ref{thm:main} 
is characterized by the $L^p$-norm for $p\in [1,\infty)$. In fact, we can extend such a result to a similar one measured in the $L^\infty$-norm.
\begin{theorem}
	\label{thm:mainInfty}
	Given any $n\in\N^+$, there exist $\bmphi_1\in \calH_{ 3^d2^{dn+5} }(d,{3^d})$ and $\phi_2\in \calH_{3^d2^{dn+8}n}( {3^dn},1)$ such that: For any $f\in C([0,1]^d)$, there exists a linear map $\calL:\R^{3^d}\to \R^{3^dn}$ satisfying
	\begin{equation*}
\begin{split}
	&\quad\  \big\|s\cdot(\phi_2\circ \calL \circ\bmphi_1)+b-f\big\|_{L^\infty([0,1]^d)}\\
	&\le \omega_f(\sqrt{d}\, 2^{-n})+2^{-n+2}\omega_f(\sqrt{d}),
\end{split}
	\end{equation*}
	where $s=2\omega_f(\sqrt{d})$, $b=f(\bmzero)-\omega_f(\sqrt{d})$, and $\calL$ is given by \[\calL(x_1,\cdots,x_{3^d})=\Big(\calL_0(x_1),\cdots,\calL_0(x_{3^d})\Big)\]
	for any $\bmx=(x_1,\cdots,x_{3^d})\in \R^{3^d}$,
	where $\calL_0:\R\to \R^n$ is a linear map given by $\calL_0(t)=(a_1t, a_2 t,\cdots,a_n t)$ with $a_1,a_2,\cdots,a_n\in [0,\tfrac{1}{3})$ determined by $f$ and $n$.
\end{theorem}


Simplifying the implicit approximation error in Theorem~\ref{thm:main} (or \ref{thm:mainInfty}) to make it explicitly depending on $n$ is challenging in general, since the modulus of continuity $\omega_f(\cdot)$ may be complicated. However, the error can be simplified if $f$ is a H{\"o}lder continuous function on $[0,1]^d$ of order $\alpha\in(0,1]$ with a H\"older constant $\lambda>0$.  That is, $f$ satisfies
\begin{equation*}\label{eqn:Holder}
	|f(\bmx)-f(\bmy)|\leq \lambda \|\bmx-\bmy\|_2^\alpha\quad \tn{for any $\bmx,\bmy\in[0,1]^d$,}
\end{equation*}
implying $\omega_f(r)\le \lambda\cdot r^\alpha$ for any $r\ge 0$. This means we can get an exponentially small approximation error $5\lambda d^{\alpha/2}2^{-\alpha n}$. In particular, in the special case of $\alpha=1$, i.e., $f$ is a Lipschitz continuous function with a Lipschitz constant $\lambda>0$, then the approximation error  is further simplified to $5\lambda\sqrt{d}\,2^{-n}$.

Though the linear mapping $\calL$ in Theorem~\ref{thm:mainInfty} is essentially determined by $n$ key parameters $a_1,a_2,\cdots,a_n$, these $n$ key parameters are repeated $3^d$ times in the final network architecture as shown in Figure~\ref{fig:phi123}. Therefore, $s\cdot(\phi_2\circ \calL \circ\bmphi_1)+b$ can be implemented by a ReLU network with $3^d n+2$ intrinsic parameters. 
Remark that we can reduce the number of intrinsic parameters to $n+2$ via using a fixed ReLU network to copy $n$ key parameters $3^d$ times. 


Furthermore, the number of intrinsic parameters can be reduced to three in the case of H\"older continuous functions. In other words, three intrinsic parameters are enough to achieve an arbitrary pre-specified error if sufficiently high precision is provided, as shown in the theorem below.


\begin{theorem}
	\label{thm:main:three:parameters}
	Given any $\varepsilon>0$, $\alpha\in (0,1]$, and $\lambda>0$, there exists $\phi\in\calH(d+1,1)$  such that: For any H\"older continuous function $f$ on $[0,1]^d$ of order $\alpha\in (0,1]$ with a H\"older constant $\lambda>0$, there exist three parameters $s\in [0,\infty)$, $v\in [0,1)$, and $b\in\R$ satisfying
	\begin{equation*}
		\big|s\phi(\bmx,v)+b-f(\bmx)\big| \le  \varepsilon\quad \tn{for any $\bmx\in[0,1]^d$.}
	\end{equation*}
\end{theorem}

In Theorem~\ref{thm:main:three:parameters}, $s$ is a scale factor, $b$ is the bias for a vertical shift, and $v$ is the key intrinsic parameter storing sufficient information of the target function $f$.
Clearly, $\phi$ is independent of  $f$,  while $s$, $v$, and $b$ are determined by $f$.
Let $\phi_2=\phi$, $\bmphi_1$ be the identity map on $\R^d$, and $\calL_v:\R^d\to\R^{d+1}$ be an affine linear transform mapping $\bmx$ to $(\bmx,v)$. Then $s\phi(\bmx,v)+b$ can also be represented as  $s\phi_2\circ\calL_v\circ\bmphi_1(\bmx)+b$, which is a special case of the architecture in \eqref{eq:basic:arc}.

Remark that Theorem~\ref{thm:main:three:parameters} is just a theoretical result since the key intrinsic parameter $v$ requires extremely high precision, which is necessary for storing the values of $f$ at sufficiently many points within a sufficiently small error. 
By using the idea of the binary representation,
we can extract the values of $f$ stored in $v$ via an $f$-independent ReLU network (as a sub-network of the final network realizing $\phi$ in Theorem~\ref{thm:main:three:parameters}).
In fact,
there is a balance between the precision requirement and the number of intrinsic parameters. For example, if we store the values of $f$ in two intrinsic parameters (not one), then the precision requirement is greatly lessened.

\subsection{Further interpretation}\label{sec:further:interpretation}

We will connect our theoretical results to
related existing results for a deeper understanding.
First, we connect our results to transfer learning. Next, we discuss the error analysis of deep neural networks to reveal
the motivation for reducing the number of parameters that need to be trained. Finally, we discuss related work from an approximation perspective.

\subsubsection{Connection to transfer learning}

Transfer learning dates back to 1970s \cite{transfer1,transfer2}.
It is a research direction in machine learning that
applies knowledge gained in one problem to solve a different but related problem.  Typically in deep learning, transfer learning uses a  pre-trained neural network obtained for one task as an initial guess of the neural network for another task to achieve a short training time. 
Our theory in this paper could provide insights into the success of transfer learning using neural networks, though the setting of our theory is different from realistic transfer learning. In our theory, the inner-function and outer-function are universally used for all learning tasks for continuous functions, which can be understood as the part of networks that can be transferred to different tasks. Suppose $f_1$ and $f_2$ are the target functions for two different but related tasks. If $f_1$ has been learned via an architecture $\phi_{f_1,\calR}\circ\bmphiout\circ\bmphi_{f_1}\circ\bmphiin$, then we can ``transfer'' the prior knowledge ($\bmphiin$ and $\bmphiout$) to another task. This means that, by only learning $\bmphi_{f_2}$ and $\phi_{f_2,\calR}$ from $f_2$, we can use $\phi_{f_2,\calR}\circ\bmphiout\circ\bmphi_{f_2}\circ\bmphiin$ to approximate $f_2$ well. Therefore, the total number of parameters that need to be learned again is not large if most of the parameters are distributed in the sub-networks corresponding to $\bmphiin$ and $\bmphiout$. 
Our theory may provide a certain theoretical understanding in the spirit of transfer learning from a network approximation perspective. To gain a deeper understanding, one can refer to \cite{doi:10.1177/1059712318818568,NIPS1992_67e103b0,Baxter1998,Mihalkova2007,pmlr-v2-niculescu-mizil07a,TL1,TL2}.
We will test the proposed architecture in \eqref{eq:basic:arc} in the context of transfer learning in Section~\ref{sec:experiments}.

\subsubsection{Error analysis}
Let us discuss the motivation for reducing the number of parameters that need to be trained. To this end, let us first talk about the error analysis of deep neural networks.
Suppose that a target function space $\scrF$ is given. To numerically compute (or approximate) the element in $\scrF$, we need to design a simple and computable hypothesis space $\scrH$ and use the elements in $\scrH$ to approximate those in $\scrF$.
To evaluate how well a numerical solution in $\scrH$ approximates an element in $\scrF$, we introduce three typical errors: the approximation, generalization, and optimization errors.

Given a target function $f\in \scrF$ defined on a domain $\calX$,  our goal is to learn a hypothesis function $\phi\in\scrH$ from finitely many samples $\{( \bm{x}_i,f(\bm{x}_i){ )}\}_{i=1}^n$. 
To infer $f(\bmx)$  for an unseen sample $\bmx$, we need to identify the empirical risk minimizer $\phi_\calS$, which is given by
\begin{equation}\label{eqn:emloss}
	\phi_{\mathcal{S}}
	 \in 
	 \argmin_{\phi\in \scrH}R_{\mathcal{S}}(\phi),
\end{equation}
where $R_{\mathcal{S}}(\phi)$ is the empirical risk defined by 
\begin{equation*}
 R_{\mathcal{S}}(\phi):=
	\frac{1}{n}\sum_{i=1}^n \ell\big( \phi(\bm{x}_i),f(\bm{x}_i)\big) 
\end{equation*}
with a loss function typically taken as $\ell(y,y')=\frac{1}{2}|y-y'|^2$.

In fact, the best hypothesis function to infer $f(\bm{x})$ is $\phi_\calD(\bm{x})$, but not $\phi_\calS(\bm{x})$, where $\phi_\calD$ is the expected risk minimizer given by
\begin{equation*}
	\phi_\calD
	\in 
	\argmin_{\phi\in \scrH} R_{\calD}(\phi),
	\end{equation*}
\tn{where} $R_{\mathcal{D}}(\phi)$ is the expected risk defined by  
\begin{equation*}
    R_{\mathcal{D}}(\phi)\coloneqq
	\mathbb{E}_{\bm{x}\sim U(\calX)} \left[\ell\big( \phi(\bm{x}),f(\bm{x})\big)\right],
\end{equation*}
where $U$ is a unknown data distribution over $\calX$.
The best possible inference error is $R_{\mathcal{D}}(\phi_{\mathcal{D}})$. In real-world applications, $U(\calX)$ is unknown and only finitely many samples from this distribution are available. Hence, the empirical risk $R_{\mathcal{S}}(\phi)$ is minimized, hoping to obtain $\phi_\calS(\bm{x})$, instead of minimizing the expected risk $R_{\mathcal{D}}(\phi)$ to obtain $\phi_\calD(\bm{x})$. A numerical optimization method to solve \eqref{eqn:emloss} may result in a numerical solution (denoted as $\phi_\calN$) that may not be a global minimizer $\phi_\calS$. Therefore, the actually learned hypothesis function to infer $f(\bm{x})$ is $\phi_\calN(\bm{x})$ and the corresponding inference error is measured by $R_{\mathcal{D}}(\phi_\calN)$, which is bounded by
\begin{scriptsize}
\begin{equation*}
\begin{aligned}
	R_{\calD}(\phi_\calN)
	& =\underbrace{
    	[R_{\calD}(\phi_\calN)-R_{\calS}(\phi_{\calN})]
	}_{\tn{GE}}
	+\underbrace{
    	[R_{\calS}(\phi_{\calN})-R_{\calS}(\phi_{\calS})]
	}_{\tn{OE}}
	\\   &\quad\ 
	+\underbrace{
    	[R_{\calS}(\phi_{\calS})-R_{\calS}(\phi_{\mathcal{D}})]
	}_{\tn{$\le 0$ by \eqref{eqn:emloss}}} +\underbrace{
    	[R_{\calS}(\phi_{\mathcal{D}})-R_{\mathcal{D}}(\phi_{\mathcal{D}})]
	}_{\tn{GE}}
	+\underbrace{
    	R_{\mathcal{D}}(\phi_{\mathcal{D}})
	}_{\tn{AE}}
	 \\ &
	 \le \underbrace{R_{\mathcal{D}}(\phi_{\mathcal{D}})}_{\tn{Approximation error (AE)}} 
	 \quad +\quad  
	 \underbrace{
    	 [R_{\calS}(\phi_{\calN})-R_{\calS}(\phi_{\calS})]
	 }_{\tn{Optimization error (OE)}}
		\\   &\quad\    + \  
		\underbrace{
    		[R_{\mathcal{D}}(\phi_{\calN})-R_{\calS}(\phi_{\calN})]
    		+[R_{\calS}(\phi_{\mathcal{D}})-R_{\mathcal{D}}(\phi_{\mathcal{D}})]
		}_{\tn{Generalization error (GE)}}.
\end{aligned}
\end{equation*}
\end{scriptsize}
As we can see from the above equation, the numerical inference error $R_\calD(\phi_\calN)$, the distance between the numerical solution $\phi_\calN$ and the target function $f$,  is bounded by three errors: the approximation, generalization, and optimization errors. See Figure~\ref{fig:aoge} for an illustration.
\begin{figure}[htp!]
 	\vskip 0.1in
	\centering
	\includegraphics[width=0.45\textwidth]{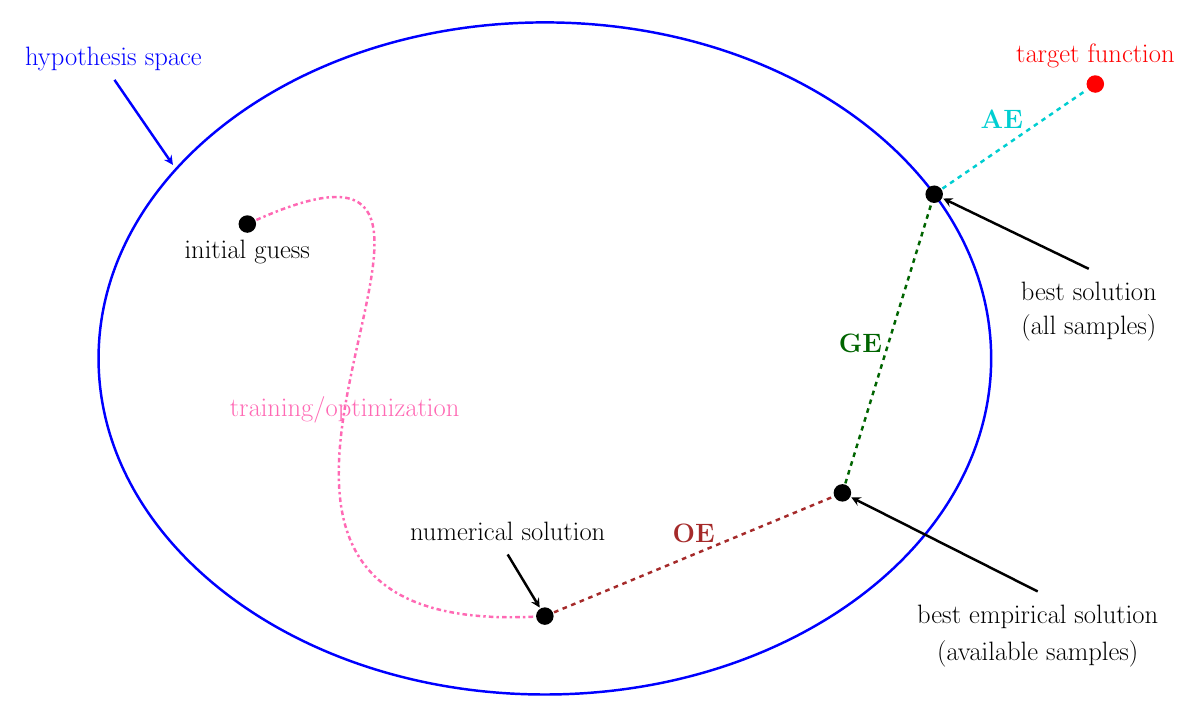}
	\caption{An illustration of the approximation error (AE), the generalization error (GE), and the optimization error (OE).}
	\label{fig:aoge}
\end{figure}

	The constructive approximation established in this paper and the literature provides an upper bound of the approximation error $R_{\mathcal{D}}(\phi_{\mathcal{D}})$. 
	 The theoretical guarantee of the convergence of an optimization algorithm to a global minimizer $\phi_{\mathcal{S}}$ and the characterization of the convergence belong to the optimization analysis of neural networks. 
	The generalization error is controlled by two key factors: the complexity of the hypothesis function space and the number of training (available) samples.
	One could refer to \cite{DBLP:journals/corr/abs-1910-00121,Weinan2019,Weinan2019APE,NIPS2016_6112,DBLP:journals/corr/NguyenH17,opt,xu2020,JMLR:v20:17-526} for further discussions of the generalization and optimization errors.
	
	Theorems~\ref{thm:main}, \ref{thm:mainInfty}, and \ref{thm:main:three:parameters} provide upper bounds of $R_{\mathcal{D}}(\phi_{\mathcal{D}})$. These bounds only depends on the number of intrinsic parameters of ReLU networks and the modulus of continuity $\omega_f(\cdot)$. Hence, these bounds are independent of the empirical risk minimization problem in \eqref{eqn:emloss} and the optimization algorithm  used to compute the numerical solution of \eqref{eqn:emloss}. In other words, Theorems~\ref{thm:main}, \ref{thm:mainInfty}, and \ref{thm:main:three:parameters} quantify the approximation power of ReLU networks in terms of the nubmer of intrinsic parameters. Designing efficient optimization algorithms and analyzing the generalization error for ReLU networks 
	are two other separate future directions.

Generally, making the hypothesis function space $\scrH$ smaller would result in a larger approximation error, a smaller generalization error, and a smaller optimization error. Thus, there is a balance for the choice of the hypothesis function space. 
Our theory pre-specifies most parameters based on the prior knowledge of the target function space $\scrF$, which leads to a smaller hypothesis function space, denoted by $\widetilde{\scrH}$. Thus, we can expect better approximation and generalization errors. Meanwhile, the approximation error becomes only a little worse since we only pre-specify unimportant parameters (non-intrinsic ones).
Therefore, we can expect a good inference (test) error by using our method to reduce the number of parameters that need to be trained.
See Figure~\ref{fig:H2tildeH} for an illustration. 
\begin{figure}[htp!]
 	\vskip 0.1in
	\centering
	\includegraphics[width=0.41\textwidth]{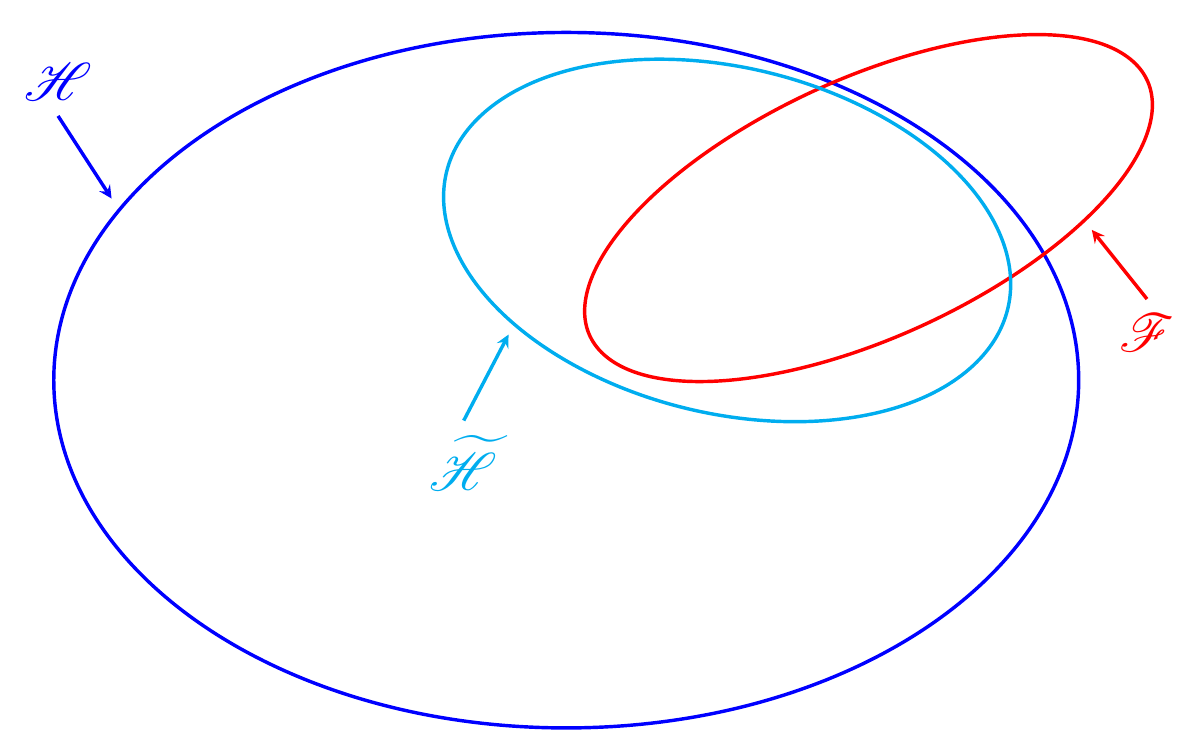}
	\caption{An illustration of our theory.
	The approximation error becomes only a little bit worse since we just pre-specify unimportant parameters based on the target function space $\scrF$.
	The new hypothesis function space $\widetilde{\scrH}$ is much smaller than the original hypothesis function space $\scrH$, leading to better generalization and optimization errors.  }
	\label{fig:H2tildeH}
\end{figure}

The key point of our method is the prior knowledge of the target function space $\scrF$. This generally means that $\scrF$ is small and special.
Our method would fail if the target function space $\scrF$ is pretty large (e.g., $\scrF=L^1([0,1]^d)$). In this case, if the new hypothesis function space $\widetilde{\scrH}$ is much smaller than the original one $\scrH$, then the approximation error would become much worse. See Figure~\ref{fig:largeF} for an illustration.

\begin{figure}[htp!]
 	\vskip 0.1in
	\centering
	\includegraphics[width=0.41\textwidth]{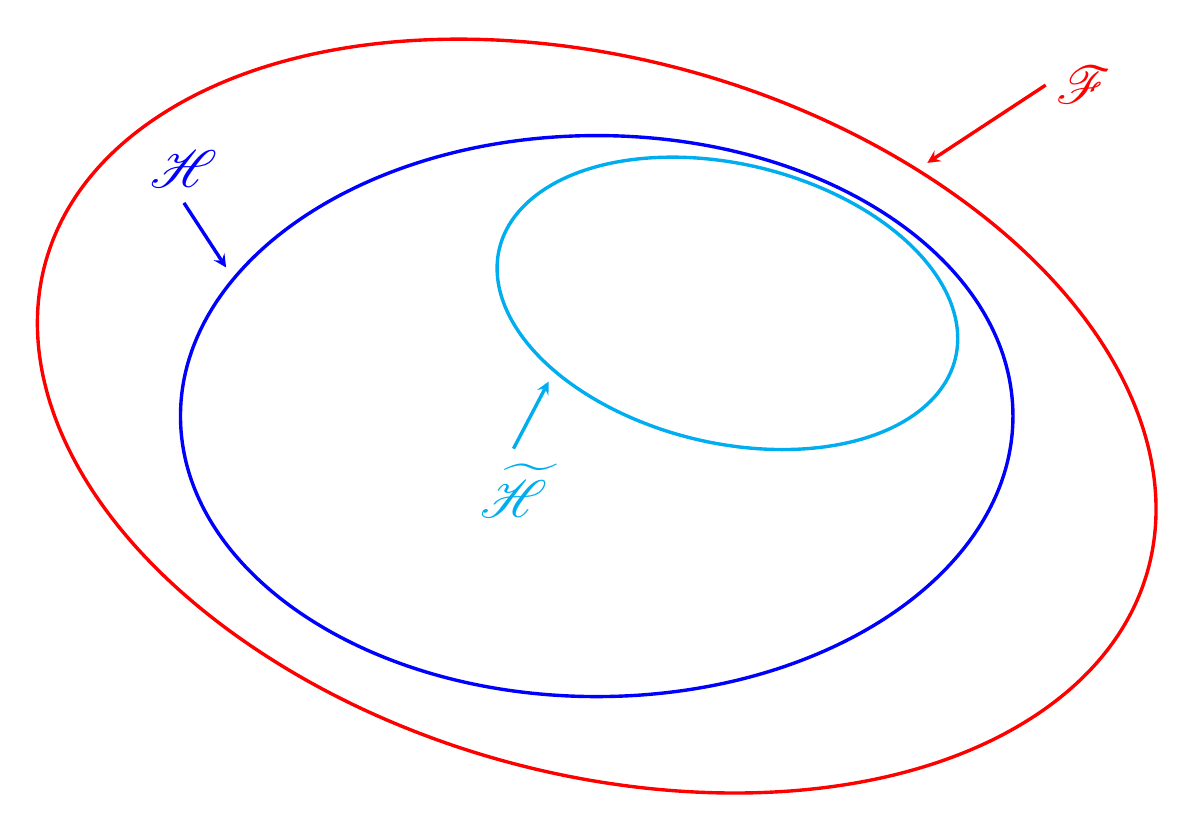}
	\caption{An illustration of the case that the target function space $\scrF$ is large.  }
	\label{fig:largeF}
\end{figure}

\subsubsection{Related work}

The expressiveness of deep neural networks has been studied extensively from many perspectives, e.g., in terms of combinatorics \cite{NIPS2014_5422}, topology \cite{ 6697897}, Vapnik-Chervonenkis (VC) dimension \cite{Bartlett98almostlinear,Sakurai,pmlr-v65-harvey17a}, fat-shattering dimension \cite{Kearns,Anthony:2009}, information theory \cite{PETERSEN2018296}, classical approximation theory \cite{Cybenko1989ApproximationBS,HORNIK1989359,barron1993,yarotsky18a,yarotsky2017,doi:10.1137/18M118709X,ZHOU2019,10.3389/fams.2018.00014,2019arXiv190501208G,2019arXiv190207896G,suzuki2018adaptivity,Ryumei,Wenjing,Bao2019ApproximationAO,2019arXiv191210382L,MO,shijun1,shijun2,shijun3,shijun:thesis}, etc. In the early works of approximation theory for neural networks, the universal approximation theorem  \cite{Cybenko1989ApproximationBS,HORNIK1991251,HORNIK1989359} without approximation errors showed that, given any $\varepsilon>0$, there exists a sufficiently large neural network approximating a target function in a certain function space within an error $\varepsilon$. For one-hidden-layer neural networks and sufficiently smooth functions, 
it is shown in \cite{barron1993,barron2018approximation} that an asymptotic approximation error $\calO(N^{-1/2})$ in the $L^2$-norm, leveraging an idea that is similar to Monte Carlo sampling for high-dimensional integrals.

Recently, it is proved in \cite{shijun2,yarotsky18a,shijun:thesis} that the (nearly) optimal approximation error would be $\calO(n^{-2/d})$ when using ReLU networks with $n$ parameters to approximate functions in the unit ball of Lipschitz continuous function space on $[0,1]^d$. Clearly, such an error suffers from the curse of dimensionality. To bridge this gap, one could either consider smaller function spaces, e.g., smooth functions \cite{shijun3,yarotsky:2019:06} and band-limited functions \cite{bandlimit}, or introducing new network architectures, e.g.,  Floor-ReLU networks \cite{shijun4}, Floor-Exponential-Step (FLES) networks \cite{shijun5}, and  (Sin, ReLU, $2^x$)-activated networks \cite{jiao2021deep}. This paper proposes a new perspective to characterize the approximation error in terms of the number of intrinsic parameters. 
Such a method is inspired by an observation that most parameters of most constructed networks in the mentioned papers are independent of the target function. Thus, most parameters can be assigned or computed in advance, i.e., we can approximate the target function by only adjusting a small number of parameters.
As shown in Theorem~\ref{thm:main}, we can first design an inner-function  $\phi_1$ and an outer-function $\phi_2$, both of which can be implemented by ReLU networks. Then, for any continuous function $f\in C([0,1]^d)$, $s\cdot(\phi_2\circ\calL\circ\phi_1)+b$ can approximate $f$ with an error $\omega_f(\sqrt{d}\, 2^{-n})+2^{-n+2}\omega_f(\sqrt{d})$ by the following two steps: 1) determining $s$ and $b$, 2) designing a linear map $\calL$ defined by $\calL(t)=(a_1t,\cdots,a_n t)$, where $a_1,\cdots,a_n$ are determined by the target function $f$. Therefore, we overcome the curse of  dimensionality in the sense of the approximation error characterized by the number of intrinsic parameters when the variation of $\omega_f(r)$ as $r\to 0$ is moderate (e.g., $\omega_f(r) \lesssim r^\alpha$ for H\"older continuous functions).


\begin{table*}[t]    
	\caption{Network architecture.} 
	\label{tab:net:architecture}
	\vskip 0.15in
	\centering  
	\resizebox{0.95\textwidth}{!}{ 
		\begin{tabular}{ccccccc} 
			\toprule
			layers    &     activation & size & dropout & batch normalization & \#parameters & remark  \\
			\midrule
			input $\in \R^{28\times 28}$  &  & $28\times 28$   \\
			
			\midrule
			
			Conv-1: $1\times (3\times 3), \, 32$ & ReLU & $(26\times 26) \times 32$ &  & yes &  $320$ & \multirow{2}{*}{low-level features, block 1 ($\bmphiin$)}   \\
			
			Conv-2: $32\times (3\times 3), \, 32$ & ReLU, MaxPool &  $(12\times 12) \times 32$ &  $0.25$ & yes & $9248$ &  \\
			
			\midrule
			
			Conv-3: $32\times (3\times 3), \, 64$ & ReLU & $(10\times 10) \times 64$ &  & yes & $18496$ &\multirow{2}{*}{high-level features, block 2 ($\bmphi_f$)} \\
			
			Conv-4: $64\times (3\times 3), \, 64$ & ReLU, Flatten & $(8\times 8) \times 64$ &  $0.25$ & yes &  $36992$ \\
			
			\midrule
			
			FC-1: $4096, \, 512$ & ReLU & $512$ &  $0.5$ & yes & $2097664$ & initial classifier, block 3 ($\bmphiout$) \\

			\midrule
			
			FC-2: $512, \, 64$ & ReLU & $64$ &   & yes & $32832$ & \multirow{2}{*}{final classifier, block 4 ($\phi_{f,\calR}$)} \\
			
			FC-3: $64, \, 10$ &  Softmax   & $10$ &  & yes & $650$  \\
			
			\midrule
			
			output $\in \R^{10}$ &   \\
			
			\bottomrule
		\end{tabular} 
	}
	\vskip -0.1in
\end{table*} 

\section{Experiments}\label{sec:experiments}
In this section, we conduct several experiments to provide numerical examples that training a small number of parameters is enough to achieve good results. We first discuss the goal of our experiments.
Next,
we extend the architecture in \eqref{eq:basic:arc} to
a simple convolutional neural
network (CNN) architecture for classification problems. Finally, we use the proposed CNN architecture to conduct several experiments and present the numerical results for three comment datasets: MNIST, Kuzushiji-MNIST (KMNIST), and Fashion-MNIST (FMNIST).

\subsection{Goal of experiments}
First, let us discuss why we adopt classification problems as our experiment examples. The goal of a classification problem with $J\in\N^+$ classes is to identify a classification function $f$ defined by 
\begin{equation*}
    f(\bmx)=j \quad \tn{for any $\bmx\in E_j$ and $j=0,1,\cdots,J-1$,}
\end{equation*}
where $E_j\subseteq E$ is the minimum closed set containing all samples with a label $j$ and $E\subseteq \R^d$ is a bounded closed set (e.g., $E=[0,1]^d$). 
Clearly, $E_0,E_1,\cdots,E_{J-1}$ are pairwise disjoint. Such a classification function can be continuously extended to $E$, which means a classification problem can also be regarded as a continuous function approximation problem.
We take the case $J=2$ as an example to illustrate the extension. The multiclass case is similar. Define 
\begin{equation*}
	\widetilde{f}(\bmx)\coloneqq \frac{\tn{dist}(\bmx,E_0)}{\tn{dist}(\bmx,E_0)+\tn{dist}(\bmx,E_1)} \quad \tn{for any $\bmx\in [0,1]^2$},
\end{equation*}
where 
\begin{equation*}
	\tn{dist}(\bmx,E_i)\coloneqq \inf_{\bmy\in E_i}\|\bmx-\bmy\|_2 
\end{equation*} 
for any $\bmx\in [0,1]^2$ and $i=0,1$.
It is easy to verify that $\widetilde{f}$ is continuous on $[0,1]^d$ and 
\begin{equation*}
    \widetilde{f}(\bmx)=f(\bmx)\quad \tn{ for any $\bmx\in E_0\cup E_1$.}
\end{equation*}
That means $\widetilde{f}$ is a continuous extension of $f$.
Remark that, in our experiments, we use networks to approximate an equivalent variant $\hatbmf$ of the original classification function $f$ mentioned above, where $\hatbmf$ is given by 
\begin{equation*}
    \hatbmf(\bmx)=\bme_j \quad \tn{for any $\bmx\in E_j$ and $j=0,1,\cdots,J-1$,}
\end{equation*}
where $\{\bme_1,\bme_2,\cdots,\bme_J\}$ is the standard basis of $\R^J$, i.e., $\bme_j\in\R^J$ denotes the vector with a $1$ in the $j$-th coordinate and $0$'s elsewhere.

In our theoretical results, we need to specify the network architecture and the values of the parameters corresponding to $\bmphiin$ and $\bmphiout$.
It is conjectured that there are more choices for the network architecture and the values of the parameters corresponding to $\bmphiin$ and $\bmphiout$. To verify that numerically, we extend the architecture in \eqref{eq:basic:arc} to a CNN architecture  and pre-specify the values of the parameters corresponding to $\bmphiin$ and $\bmphiout$ via pre-training all parameters with a part of training samples. 

Our theory only focuses on the approximation error, while the test error (accuracy) in the experiment is bounded by three errors, as discussed in Section~\ref{sec:further:interpretation}.
A good approximation error may not guarantee a good test error.
However, the approximation error is bounded by the test error, as we can see from the discussion of the error analysis previously.
Thus, a good test error implies that the approximation error is well controlled. 
It remains to show that training a small number of parameters are enough to achieve a good test error. 
If this can be numerically observed, then the proposed CNN can approximate the classification function well via only adjusting only a small number of parameters. 


\subsection{Network architecture}\label{sec:cnn:architructure}

We will design 
a CNN architecture to classify the images in three datasets: MNIST, KMNIST, and FMNIST. 
Each of these three datasets has ten classes and their elements have a size ${28\times 28}$. Thus, we can take the same CNN architecture for these three datasets.
For simplicity, we consider a basic CNN architecture: four convolutional layers and three fully connected layers. 
The whole CNN architecture is summarized in Figure~\ref{fig:CNNarc}.
We present more details of the proposed CNN architecture and connect it to the architecture in \eqref{eq:basic:arc} in Table~\ref{tab:net:architecture}.
\begin{figure}[htp!]
 	\vskip 0.1in
	\centering
	\includegraphics[width=0.472\textwidth]{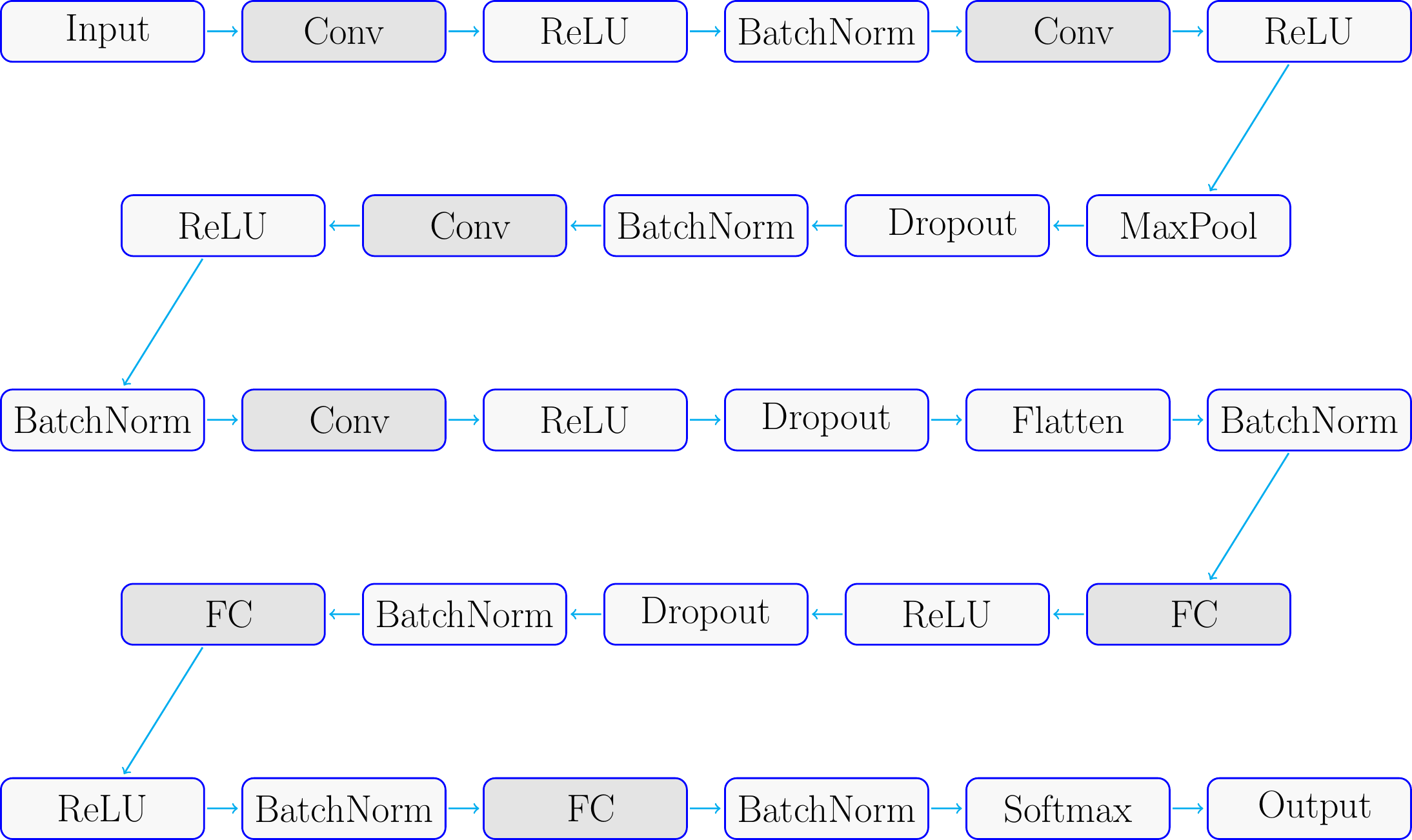}
	\caption{
	An illustration of the CNN architecture. Conv and FC are short of convolutional  and fully connected layers, respectively.  }
	\label{fig:CNNarc}
\end{figure}

To illustrate the connection between the proposed CNN architecture and the architecture in \eqref{eq:basic:arc},
we divide the proposed CNN architecture into four main blocks. Block $1$ has two convolutional layers, extracting the low-level features; Block $2$ has two convolutional layers, extracting the high-level features; 
Block $3$ has one fully connected layer, regarded as the initial classifier; Block $4$ has two fully connected layers, regarded as the final classifier.
See a summary in the following equation.
\begin{equation*}
	\begin{split}
		\tn{input}&\to
		\underbrace{
			\tn{Conv-1} \to \tn{Conv-2} 
		}_{\tn{block 1 (low-level features)}}\to
		\underbrace{
			\tn{Conv-3} \to \tn{Conv-4} 
		}_{\tn{block 2 (high-level features)}}
		\\	
		&\to
		\underbrace{
			\tn{FC-1}  
		}_{\tn{block 3 (initial classifier})}\to
		\underbrace{
			\tn{FC-2} \to \tn{FC-3} 
		}_{\tn{block 4 (final classifier})}\to
		\tn{output}
	\end{split}
\end{equation*}

Remark that the above CNN architecture can be considered as a special case of the architecture in \eqref{eq:basic:arc}. 
Blocks $1$, $2$, $3$, and $4$ in the proposed CNN architecture
are approximately equivalent to 
$\bmphiin$, $\bmphi_f$, $\bmphiout$, and $\phi_{f,\calR}$ in the architecture in \eqref{eq:basic:arc}, respectively.

\begin{figure*}[tb]
 	\vskip 0.1in
	\centering
	\begin{subfigure}[b]{0.325\textwidth}
		\centering            \includegraphics[width=0.98\textwidth]{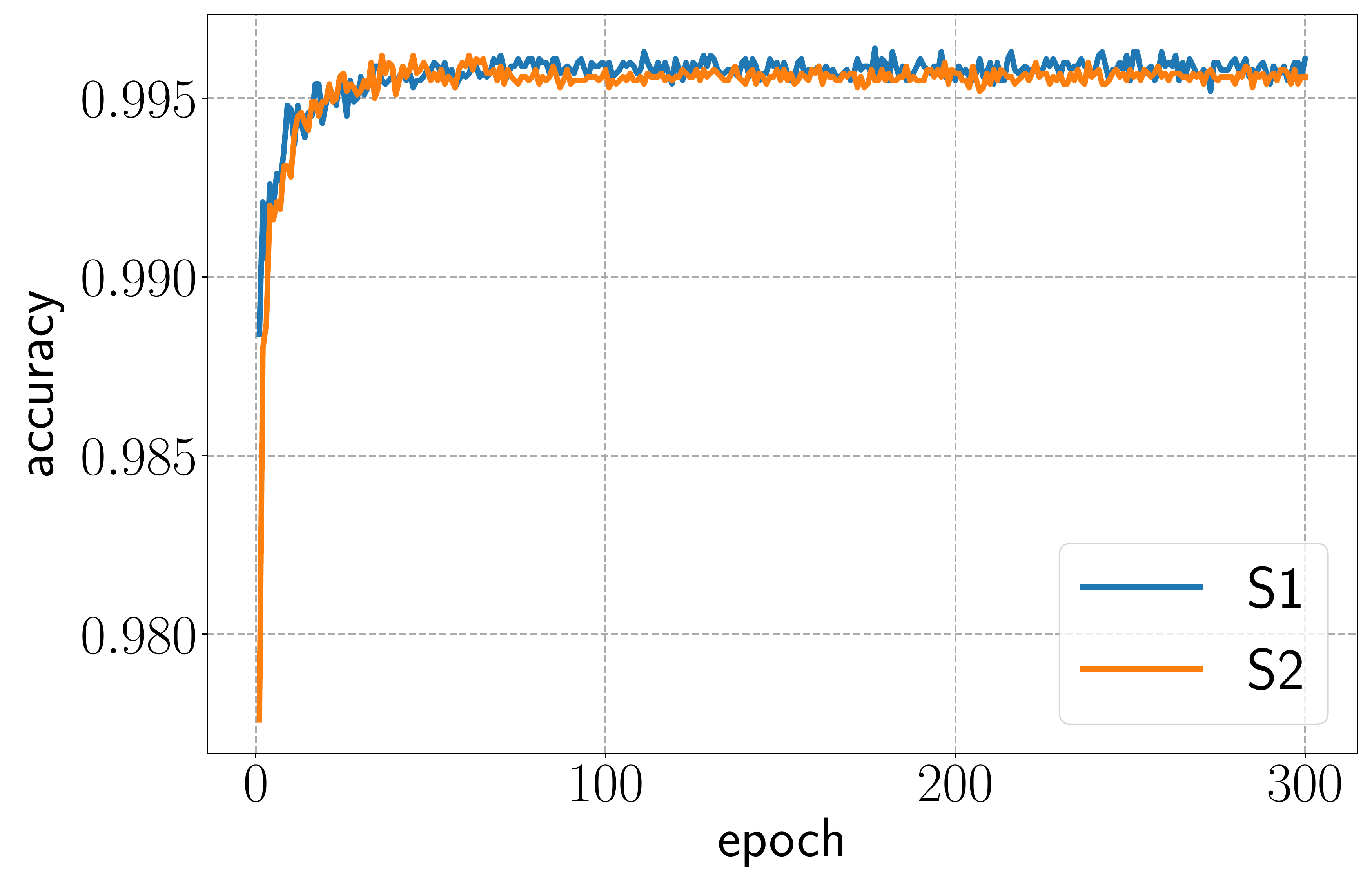}
		\subcaption{MNIST.}
	\end{subfigure}
	\begin{subfigure}[b]{0.325\textwidth}
		\centering           \includegraphics[width=0.98\textwidth]{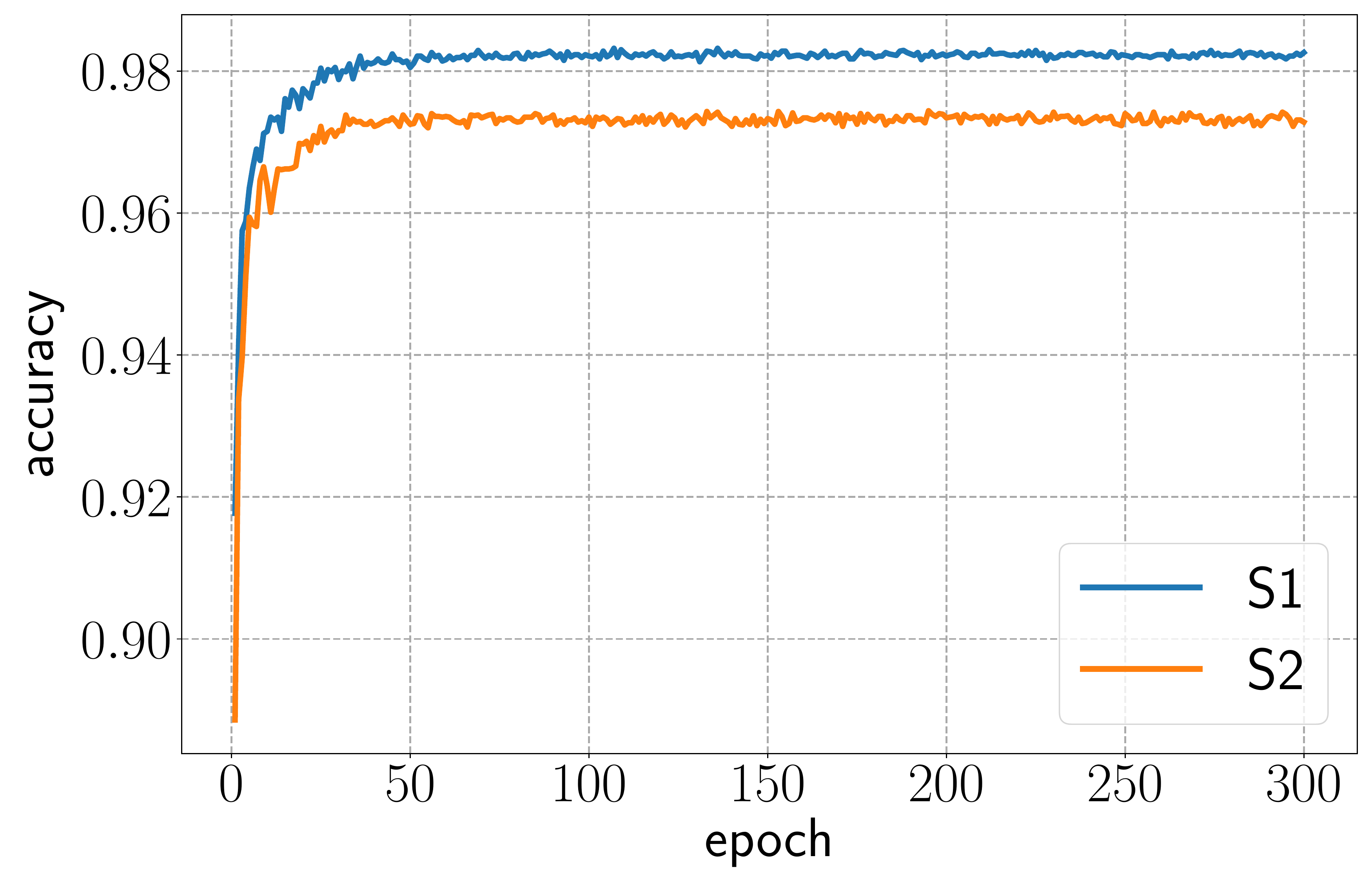}
		\subcaption{KMNIST.}
	\end{subfigure}
	\begin{subfigure}[b]{0.325\textwidth}
		\centering           \includegraphics[width=0.98\textwidth]{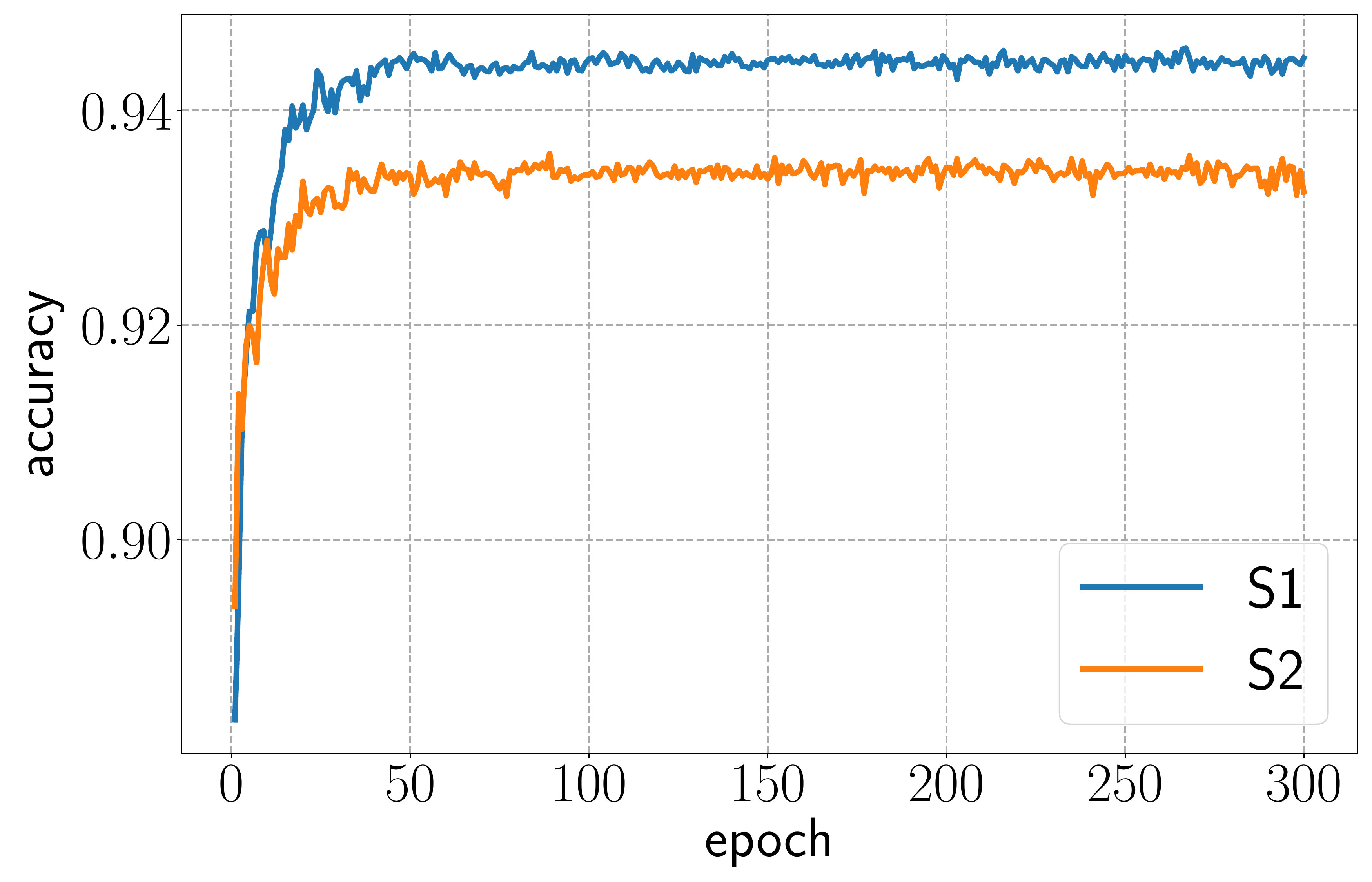}
		\subcaption{FMNIST.}
	\end{subfigure}
	\caption{Test accuracy over epochs.}
	\label{fig:accuracy:mnist}
\end{figure*}
\subsection{Numerical results}
Our goal is to numerically verify our theoretical result that adjusting a small number of parameters is enough to achieve a good approximation. 
It is natural to ask
how to pre-specify the values of the non-training parameters.  Since it is difficult to manually specify the values of the parameters in a CNN architecture, we use a part of training samples to pre-training all parameters.
Then, we propose two optimization strategies to train the proposed CNN as follows.
\begin{itemize}
	\item[(S1)] The normal strategy: We use all training samples to train all parameters.
	
	\item[(S2)] A strategy based on the architecture in \eqref{eq:basic:arc}: We first use the training samples with labels $0,1,\cdots,4$ to pre-training all parameters, and then use all training samples to continue training the parameters in Blocks $2$ and $4$.
\end{itemize}

Before presenting the numerical results, 
let us present the hyperparameters for training the proposed CNN architecture. 
To reduce overfitting and speed up optimization,
we take two common regularization methods:
dropout \cite{DBLP:journals/corr/abs-1207-0580,JMLR:v15:srivastava14a} and batch normalization \cite{10.5555/3045118.3045167}. 
We use the cross-entropy loss function to evaluate the loss. 
The number of epochs and the batch size are set to $300$ and $128$, respectively.
We adopt 
RAdam \cite{Liu2020On}
as the
optimization method. The weight decay of the optimizer is $0.0001$ and the learning rate is $0.002\times0.93^{i-1}$ in the $i$-th epoch. 
Remark that all training (test) samples are standardized before training, i.e., we rescale the samples to have a mean of $0$ and a standard deviation of $1$.

For the three mentioned datasets, we use two proposed optimization strategies to train the proposed CNN and use all test samples to obtain the test accuracy.
As we can see from Figure~\ref{fig:accuracy:mnist}, the test accuracy becomes steady after $50$ epochs. Thus, it is reasonable to take 
the largest test accuracy over epochs as the target test accuracy. The test accuracy comparison of two optimization strategies is summarized in Table~\ref{tab:accuracy}.

\begin{table}[htp!]    
	\caption{Test accuracy comparison.} 
	\label{tab:accuracy}
	\vskip 0.15in
	\centering  
	\resizebox{0.47\textwidth}{!}{ 
		\begin{tabular}{ccccccccc} 
			\toprule
			strategy &  MNIST   &     KMNIST & FMNIST  &  \#training-parameters  \\
			\midrule
		\rowcolor{mygray}	(S1) & 0.9964 &  0.9832 & 0.9458 &   $2.2\times 10^6$ \\			
			\midrule
			 (S2) & 0.9962  & 0.9744 & 0.9360 &  $8.9\times 10^4$   \\
			\bottomrule
		\end{tabular} 
	}
\end{table} 

As we can see from Table~\ref{tab:accuracy},
for all three datasets, the second optimization strategy (S2) trains much less parameters at the price of a slightly lower test accuracy, compared to the first optimization strategy (S1). 
As discussed in Section~\ref{sec:further:interpretation}, the test  accuracy (error) is controlled by three errors: the approximation, generalization, and optimization errors.
A good approximation error cannot guarantee a good test accuracy.
However, a good test accuracy numerically implies that the three errors are all well controlled. 
Our numerical results suggest that training a small number of parameters is enough to achieve a good test accuracy. 
Therefore, the proposed CNN can approximate the classification function well via only adjusting only a small number of parameters.

\section{Conclusion}\label{sec:conclusion}
This paper aims to achieve a good approximation  via adjusting only a small number of parameters based on the target function $f$ while using a ReLU network to approximate $f$.
We first propose a composition architecture in \eqref{eq:basic:arc}, and then use such an architecture to construct ReLU networks to approximate the target function $f$.
 In Theorem~\ref{thm:main}, we prove that a ReLU network with $n+2$ intrinsic parameters can approximate a continuous function $f$ on $[0,1]^d$ with an error $\omega_f(\sqrt{d}\, 2^{-n})+2^{-n+2}\omega_f(\sqrt{d})$ measured in the $L^p$-norm for $p\in [1,\infty)$.
Moreover, such a result can be generalized from the $L^p$-norm to the $L^\infty$-norm at a price of adding $\calO(n)$ intrinsic parameters, as shown in Theorem~\ref{thm:mainInfty}. Next, we show in Theorem~\ref{thm:main:three:parameters} that three intrinsic parameters are enough to achieve an arbitrarily small error
in the case of H\"older continuous functions, though this result requires high precision to encode these three parameters on computers.
Finally, we conduct several experiments to provide numerical examples of
the architecture in \eqref{eq:basic:arc}.
Remark that this paper only focuses on the approximation error characterized by the number of intrinsic parameters, the study of the optimization error and generalization error will be left as future work. 

\section*{Acknowledgments}
Z.~Shen is supported by 
Distinguished Professorship of National University of Singapore.
H.~Yang was partially supported by the US National Science Foundation under award DMS-1945029. 


\bibliography{references}
\bibliographystyle{icml2022}

\newpage
\appendix
\onecolumn

\section{Proofs of main theorems}
\label{sec:proofMain}
In this section, we first list all notations used throughout this paper. Then, we prove Theorems~\ref{thm:main}, \ref{thm:mainInfty}, and \ref{thm:main:three:parameters} based on an auxiliary theorem, Theorem~\ref{thm:mainOld}, which will be proved in Section~\ref{sec:proofMainOld}.

\subsection{Notations}
Let us summarize the main notations as follows.
\begin{itemize}	  
	\item Let $\R$, $\Q$, and $\Z$ denote the set of real numbers, rational numbers, and integers, respectively.
	
	\item Let $\N$ and $\N^+$ denote the set of natural numbers and positive natural numbers, respectively.  That is,
	$\N^+=\{1,2,3,\cdots\}$ and $\N=\N^+\bigcup\{0\}$.
	
	\item Vectors and matrices are denoted in a bold font. Standard vectorization is adopted in the matrix and vector computation. For example, adding a scalar and a vector means adding the scalar to each entry of the vector.
	
	\item For $\theta\in[0,1)$, suppose its binary representation is $\theta=\sum_{\ell=1}^{\infty}\theta_\ell2^{-\ell}$ with $\theta_\ell\in \{0,1\}$, we introduce a special notation $\bin 0.\theta_1\theta_2\cdots \theta_L$ to denote the $L$-term binary representation of $\theta$, i.e., ${\bin 0.\theta_1\theta_2\cdots \theta_L\coloneqq}\sum_{\ell=1}^{L}\theta_\ell2^{-\ell}$. 
	
	\item For any $p\in [1,\infty)$, the $p$-norm of a vector $\bmx=(x_1,x_2,\cdots,x_d)\in\R^d$ is defined by 
	\begin{equation*}
		\|\bmx\|_p\coloneqq \big(|x_1|^p+|x_2|^p+\cdots+|x_d|^p\big)^{1/p}.
	\end{equation*}
	
	\item The expression ``a network with (of) width $N$ and depth $L$'' means
	\begin{itemize}
		\item The maximum width of this network for all \textbf{hidden} layers  is no more than $N$.
		\item The number of \textbf{hidden} layers of this network is  no more than $L$.
	\end{itemize}
	
	
	\item  Similar to ``$\min$'' and ``$\max$'', let $\middleValue(x_1,x_2,x_3)$ be the middle value (median) of three inputs $x_1$, $x_2$, and $x_3$. For example, $\middleValue(2,1,3)=2$ and $\middleValue(3,2,3)=3$.

	\item 
	Given any $K\in \N^+$ and $\delta\in(0,\tfrac{1}{K}]$, define a trifling region  $\Omega([0,1]^d,K,\delta)$ of $[0,1]^d$ as 
	\begin{equation}
		\label{eq:badRegionDef}
		\Omega([0,1]^d,K,\delta)\coloneqq\bigcup_{i=1}^{d} \Big\{\bmx=(x_1,x_2,\cdots,x_d): x_i\in \cup_{k=1}^{K-1}(\tfrac{k}{K}-\delta,\tfrac{k}{K})\Big\}.
	\end{equation}
	In particular, $\Omega([0,1]^d,K,\delta)=\emptyset$ if $K=1$. See Figure~\ref{fig:region} for two examples of trifling regions.
	\begin{figure}[ht!]        
		\centering
		\begin{subfigure}[b]{0.33\textwidth}
			\centering            \includegraphics[width=0.75\textwidth]{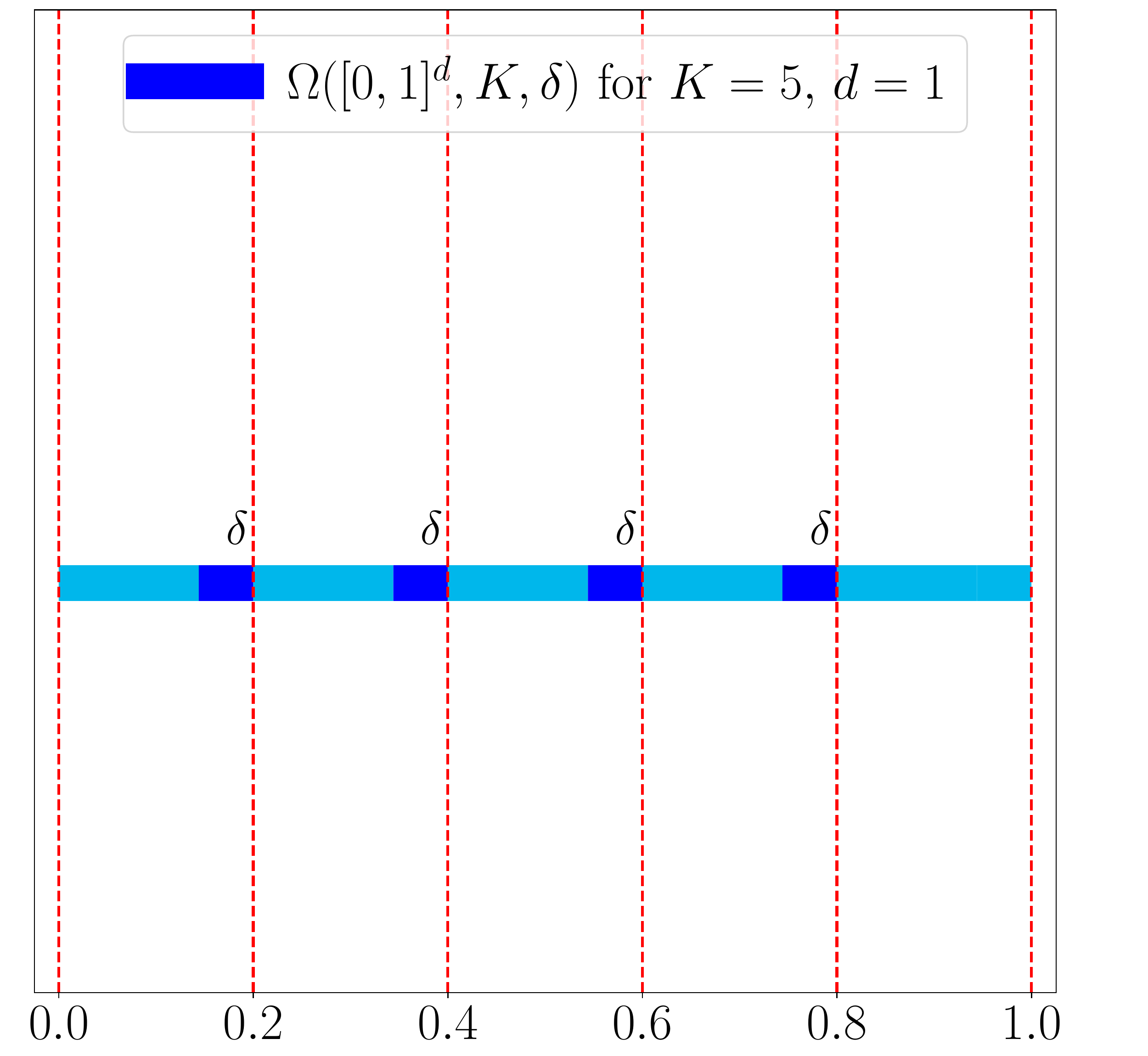}
			\subcaption{}
		\end{subfigure}
		\begin{subfigure}[b]{0.33\textwidth}
			\centering            \includegraphics[width=0.75\textwidth]{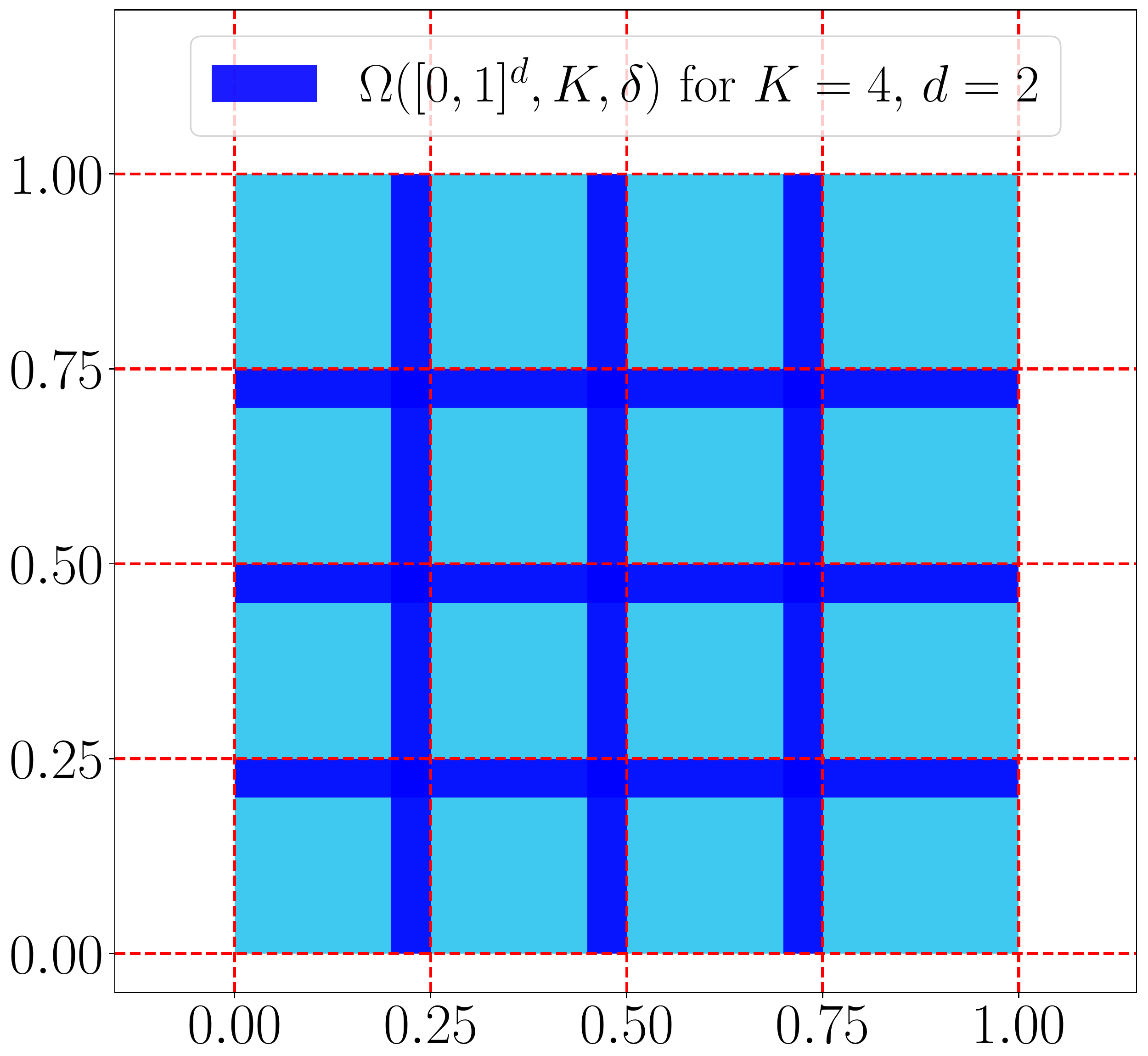}
			\subcaption{}
		\end{subfigure}
		\caption{Two examples of trifling regions. (a)  $K=5,d=1$. (b) $K=4,d=2$.}
		\label{fig:region}
	\end{figure}
	
	\item 
	Given a univariate  activation function $\sigma$, let us introduce the architecture of a $\sigma$-activated network, i.e.,  a network with each hidden neuron activated by $\sigma$.
	To be precise, a $\sigma$-activated network with a vector input $\bmx\in\R^d$, an output $\phi(\bmx)\in\R$, and $L\in\N^+$ hidden layers can be briefly described as follows:
	\begin{equation}\label{eq:Phi:x:theta}
		\begin{aligned}
			\bm{x}=\widetilde{\bm{h}}_0 
			\  \myto{2.0}^{\bmA_0,\ \bm{b}_0}_{\calL_0} 
			\  \bm{h}_1
			\  \myto{1.1}^{\sigma} 
			\  \widetilde{\bm{h}}_1 
			\quad  \cdots\quad 
			\  \myto{2.7}^{\bmA_{L-1},\ 
			  \bm{b}_{L-1}}_{\calL_{L-1}}
			  \   \bm{h}_L
			  \   \myto{1.1}^{\sigma}
			  \  \widetilde{\bm{h}}_L
			\   \myto{2.0}^{\bmA_{L},\ \bm{b}_{L}}_{\calL_L}
			\   \bm{h}_{L+1}=\phi(\bm{x}),
		\end{aligned}
	\end{equation}
	where $N_0=d\in\N^+$, $N_1,N_2,\cdots,N_L\in\N^+$,  $N_{L+1}=1$, $\bmA_i\in \R^{N_{i+1}\times N_{i}}$ and $\bm{b}_i\in \R^{N_{i+1}}$ are the weight matrix and the bias vector in the $i$-th affine linear transform $\calL_i$, respectively, i.e., 
	\[\bm{h}_{i+1} =\bmA_i\cdot \widetilde{\bm{h}}_{i} + \bm{b}_i\eqqcolon \calL_i(\widetilde{\bm{h}}_{i})\quad \tn{for $i=0,1,\cdots,L$}\]  
	and
	\[
	\widetilde{{h}}_{i,j}= \sigma({h}_{i,j})\quad \tn{for $j=1,2,\cdots,N_i$ and $i=1,2,\cdots,L$.}
	\]
	Here, $\widetilde{{h}}_{i,j}$ and ${h}_{i,j}$ are the $j$-th entry of $\widetilde{\bm{h}}_i$ and $\bm{h}_i$, respectively, for
	$j=1,2,\cdots,N_i$ and $i=1,2,\cdots,L$.
	If $\sigma$ is applied to a vector entrywisely, i.e., 
	\begin{equation*}
		\sigma(\bmy)=\big(\sigma(y_1),\cdots,\sigma(y_d)\big)\quad \tn{for any $\bmy=(y_1,\cdots,y_d)\in\R^d$,}
	\end{equation*}
	then $\phi$ can be represented in a form of function compositions as follows:
	\begin{equation*}
		\phi(\bmx) =\calL_L\circ\sigma\circ 
		\quad  \cdots \quad  
		\circ \sigma\circ\calL_1\circ\sigma\circ\calL_0(\bmx)\quad \tn{for any $\bmx\in\R^d$}.
	\end{equation*}
	See Figure~\ref{fig:sigma:net:eg} for an example.
	
	\begin{figure}[ht!]
		\centering
		\includegraphics[width=0.68\textwidth]{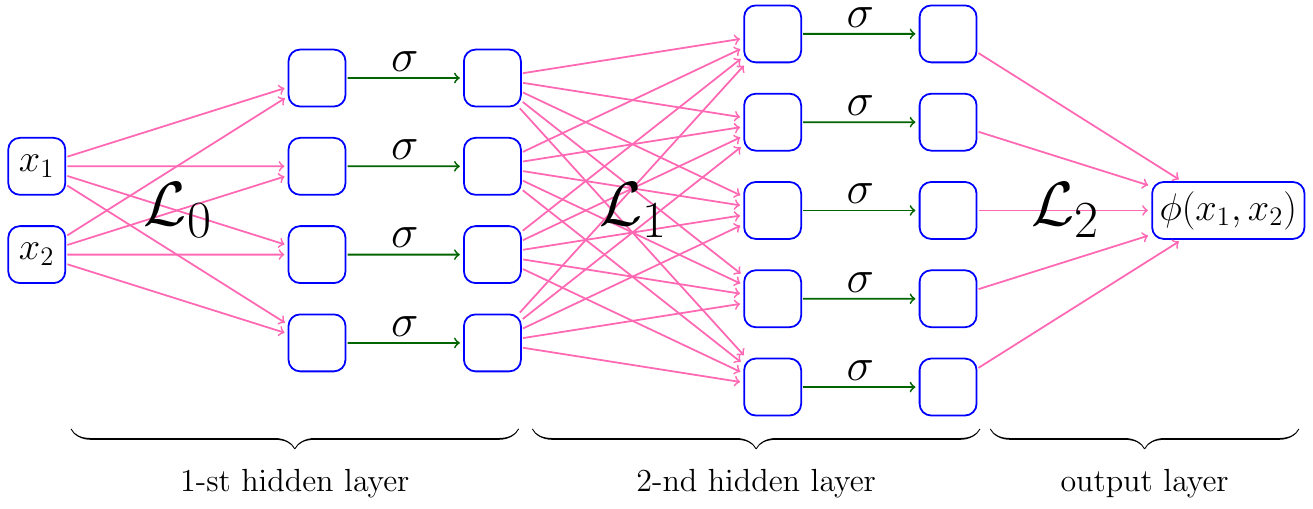}
		\caption{An example of a $\sigma$-activated network with width $5$ and depth $2$.}
		\label{fig:sigma:net:eg}
	\end{figure}
	
\end{itemize}

\subsection{Proof of Theorem~\ref{thm:main}}
To prove  Theorems~\ref{thm:main} and \ref{thm:mainInfty}, we introduce an auxiliary theorem below with a similar result  ignoring the approximation inside the trifling region.
\begin{theorem}
	\label{thm:mainOld}
	Given any $n\in\N^+$, there exist  $\phi_1\in \calH_{ 2^{dn+4} }(d,1)$ and $\phi_2\in \calH_{ 2^{dn+5}n }(n,1)$ such that: For any continuous function $f:[0,1]^d\to [0,1]$, there exists a linear map $\calL:\R\to \R^n$ satisfying $\|\phi_2\circ \calL \circ\phi_1\|_{L^\infty(\R^d)}\le 1$  and
	\begin{equation*}
		\big|\phi_2\circ \calL \circ\phi_1(\bmx)-f(\bmx)\big|\le \omega_f(\sqrt{d}\, 2^{-n})+2^{-n}\quad \tn{for any $\bmx\in [0,1]^d\backslash \Omega([0,1]^d,K,\delta)$,}
	\end{equation*}
	where $K=2^n$, $\delta$ is an arbitrary number in $(0,\tfrac{1}{3K}]$, and $\calL$ is given by $\calL(t)=(a_1t, a_2 t,\cdots,a_n t)$. Here, $a_i\in [0,\tfrac{1}{3})$ is determined by $f$ and $n$, and $2^m a_i\in\N$ for $i=1,2,\cdots,n$, where $m=2^{dn+1}$.
\end{theorem}
The proof of Theorem~\ref{thm:mainOld} can be found later in Section~\ref{sec:proofMainOld}.
Let us first prove Theorem~\ref{thm:main} based on Theorem~\ref{thm:mainOld}.
\begin{proof}[Proof of Theorem~\ref{thm:main}]
	We may assume $f$ is not a constant function  since it is a trivial case. 
	Then $\omega_f(r)>0$ for any $r>0$. Set $s=2\omega_f(\sqrt{d})>0$ and $b=f(\bmzero)-\omega_f(\sqrt{d})$. Then, by defining 
	\begin{equation*}
		\tildef\coloneqq \frac{f-b}{s}=\frac{f-f(\bmzero)+\omega_f(\sqrt{d})}{2\omega_f(\sqrt{d})},
	\end{equation*}
	we have $\tildef(\bmx)\in [0,1]$ for any $\bmx\in [0,1]^d$. 
	By applying Theorem~\ref{thm:mainOld} to $\tildef$, there exist two  functions, $\phi_1:\R^d\to \R$ and $\phi_2:\R^n\to \R$, both of which are independent of $f$ and can be implemented by ReLU networks with $\le 2^{dn+4}$ and $\le 2^{dn+5}n$ parameters, respectively, and a linear map $\calL:\R\to \R^n$ satisfying $\|\phi_2\circ \calL \circ\phi_1\|_{L^\infty(\R^d)}\le 1$  and
	\begin{equation*}
		|\phi_2\circ \calL \circ\phi_1(\bmx)-\tildef(\bmx)|\le \omega_\tildef(\sqrt{d}\, 2^{-n})+2^{-n}\quad \tn{for any $\bmx\in [0,1]^d\backslash \Omega([0,1]^d,K,\delta)$,}
	\end{equation*}
	where $K=2^n$, $\delta$ is an arbitrary number in $(0,\tfrac{1}{3K}]$, and $\calL$ is given by $\calL(t)=(a_1t, a_2 t,\cdots,a_n t)$ with $a_1,a_2,\cdots,a_n\in [0,\tfrac{1}{3})$ determined by $\tildef$ and $n$. Since $\tildef$ is derived from $f$, $a_1,a_2,\cdots,a_n$ are essentially determined by $f$ and $n$.
	
	Then choose a small $\delta$ satisfying 
	\begin{equation*}
		dK\delta 2^p \le 2^{-pn}= (2^{-n})^p.
	\end{equation*}
	
	Note that the Lebesgue measure of $\Omega([0,1]^d,K,\delta)$ is bounded by $dK\delta$ and
	\begin{equation*}
		|\phi_2\circ\calL\circ\phi_1(\bmx)-\tildef(\bmx)|\le |\phi_2\circ\calL\circ\phi_1(\bmx)|+|\tildef(\bmx)|\le 1+1=2\quad \tn{for any $\bmx\in [0,1]^d$.}
	\end{equation*}
	Then, $\|\phi_2\circ\calL\circ\phi_1-\tildef\|_{L^p([0,1]^d)}^p$ is bounded by
	\begin{equation*}
		\begin{split}
			&\quad \int_{[0,1]^d\backslash \Omega([0,1]^d,K,\delta)}|\phi_2\circ\calL\circ\phi_1(\bmx)-\tildef(\bmx)|^p d\bmx
			+\int_{ \Omega([0,1]^d,K,\delta)}|\phi_2\circ\calL\circ\phi_1(\bmx)-\tildef(\bmx)|^p d\bmx\\
			&\le \big(\omega_\tildef(\sqrt{d}\, 2^{-n})+2^{-n}\big)^p+ dK\delta 2^p
			\le \big(\omega_\tildef(\sqrt{d}\, 2^{-n})+2^{-n}\big)^p + (2^{-n})^p\le \big(\omega_\tildef(\sqrt{d}\, 2^{-n})+2^{-n+1}\big)^p,
		\end{split}
	\end{equation*}
	implying $\|\phi_2\circ\calL\circ\phi_1-\tildef\|_{L^p([0,1]^d)}\le \omega_\tildef(\sqrt{d}\, 2^{-n})+2^{-n+1}$. Note that $\omega_f(r)=s \cdot \omega_\tildef(r)$ for any $r\ge 0$. Therefore, we have 
	\begin{equation*}
		\begin{split}
			&\quad \|s\cdot(\phi_2\circ \calL \circ \phi_1)+b -f\|_{L^p([0,1]^d)}
			=\Big\|s\cdot (\phi_2\circ \calL \circ \phi_1)+ b-\big(s\cdot \tildef+b\big)\Big\|_{L^p([0,1]^d)}\\
			&=s\|\phi_2\circ\calL\circ\phi_1-\tildef\|_{L^p([0,1]^d)}\le s\cdot \omega_\tildef(\sqrt{d}\, 2^{-n})+ 2^{-n+1}  s
			= \omega_f(\sqrt{d}\,2^{-n})+2^{-n+2}\omega_f(\sqrt{d}).
		\end{split}
	\end{equation*}
	So we finish the proof.
\end{proof}

\subsection{Proof of Theorem~\ref{thm:mainInfty}}
Next, let us prove Theorem~\ref{thm:mainInfty}. To this end, we need to introduce the following lemma, which is actually Lemma $3.4$ of \cite{shijun3} (or Lemma~$3.11$ of \cite{shijun:thesis}).
\begin{lemma}[Lemma $3.4$ of \cite{shijun3}]
	\label{lem:approx:trifling}
	Given any $\varepsilon>0$, $K\in \N^+$, and $\delta\in (0,\tfrac{1}{3K}]$,
	assume $f\in C([0,1]^d)$ and $g:\R^d\to\R$ is a general function with 
	\begin{equation*}
		|g(\bmx)-f(\bmx)|\le \varepsilon\quad \tn{for any $\bmx\in [0,1]^d\backslash \Omega([0,1]^d,K,\delta)$.}
	\end{equation*} 
	Then 
	\begin{equation*}
		|\phi(\bmx)-f(\bmx)|\le \varepsilon+d\cdot\omega_f(\delta)\quad \tn{for any $\bmx\in [0,1]^d$,}
	\end{equation*}
	where $\phi\coloneqq  \phi_d$ is defined by induction through
	\begin{equation*}
		\label{eq:phiInduction}
		\phi_{i+1}(\bmx)\coloneqq  \middleValue\big(\phi_{i}(\bmx-\delta\bme_{i+1}),\phi_{i}(\bmx),\phi_{i}(\bmx+\delta\bme_{i+1})\big)\quad \tn{for $i=0,1,\cdots,d-1$,}
	\end{equation*}
	where $\phi_0=g$ and $\{\bme_i\}_{i=1}^d$ is the standard basis in $\mathbb{R}^d$. 	
\end{lemma}

With Lemma~\ref{lem:approx:trifling} in hand, we are ready to present the proof of Theorem~\ref{thm:mainInfty}.
\begin{proof}[Proof of Theorem~\ref{thm:mainInfty}]
	We may assume $f$ is not a constant function  since it is a trivial case. 
	Then $\omega_f(r)>0$ for any $r>0$. Set $s=2\omega_f(\sqrt{d})>0$ and $b=f(\bmzero)-\omega_f(\sqrt{d})$. Then, by defining 
	\begin{equation*}
		\tildef\coloneqq \frac{f-b}{s}=\frac{f-f(\bmzero)+\omega_f(\sqrt{d})}{2\omega_f(\sqrt{d})},
	\end{equation*}
	we have $\tildef(\bmx)\in [0,1]$ for any $\bmx\in [0,1]^d$. 
	By applying Theorem~\ref{thm:mainOld} to $\tildef$, there exist two functions, $\psi_{1,0}:\R^d\to \R$ and $\psi_{2,0}:\R^n\to \R$, both of which are independent of $\tildef$ (or $f$) and can be implemented by ReLU networks with $\le 2^{dn+4}$ and $\le 2^{dn+5}n$ parameters, respectively, and a linear map $\calL_{\bma,0}:\R\to \R^n$ satisfying $\|\psi_0\|_{L^\infty(\R^d)}\le 1$  and
	\begin{equation*}
		|\psi_0(\bmx)-\tildef(\bmx)|\le \omega_\tildef(\sqrt{d}\, 2^{-n})+2^{-n}\eqqcolon \varepsilon\quad \tn{for any $\bmx\in [0,1]^d\backslash \Omega([0,1]^d,K,\delta)$,}
	\end{equation*}
	where $\psi_0\coloneqq \psi_{2,0}\circ \calL_{\bma,0} \circ\psi_{1,0}$,
	$K=2^n$, $\delta$ is an arbitrary number in $(0,\tfrac{1}{3K}]$, and $\calL_{\bma,0}$ is given by $\calL_{\bma,0}(t)=(a_1t, a_2 t,\cdots,a_n t)$ with $\bma=(a_1,a_2,\cdots,a_n)$ determined by $f$ and $n$. Moreover,
	$a_i\in[0,\tfrac{1}{3})$ and $2^m a_i\in\N$ for $i=1,2,\cdots,n$, where $m=2^{dn+1}$.\footnote{This property will be used in the proof of Theorem~\ref{thm:main:three:parameters}. }
	
	Choose a small $\delta$ satisfying $d\cdot \omega_\tildef(\delta)\le 2^{-n}$.
	With $\psi_0=\psi_{2,0}\circ \calL_{\bma,0} \circ\psi_{1,0}$ in hand, we can define $\psi_1,\cdots,\psi_d$ by induction via 
	\begin{equation*}
		\psi_{i+1}(\bmx)\coloneqq \middleValue\Big(\psi_{i}(\bmx-\delta\bme_{i+1}),\psi_{i}(\bmx),\psi_{i}(\bmx+\delta\bme_{i+1})\Big)\quad \tn{for $i=0,1,\cdots,d-1$.}
	\end{equation*}
	The detailed iterative equations for $\psi_{i}:\R^d\to \R$, $\bmpsi_{1,i}:\R^d\to \R^{3^i}$, $\calL_{\bma,i}:\R^{3^i}\to \R^{3^i d}$, and $\psi_{2,i}:\R^{3^i d}\to \R$, 	for $i=1,2,\cdots,d$, are listed as follows.
	\begin{itemize}
		\item $\psi_{i}=\psi_{2,i}\circ \calL_{\bma,i}\circ \bmpsi_{1,i}$.
		
		\item $\bmpsi_{1,i}(\bmy)=\Big(\bmpsi_{1,i-1}(\bmy-\delta\bme_i),\  \bmpsi_{1,i-1}(\bmy),\   \bmpsi_{1,i-1}(\bmy+\delta\bme_i)\Big)$ \quad   for any $\bmy\in \R^d$.
		
		\item $\calL_{\bma,i}(\bmy_1,\bmy_2,\bmy_3)=\Big(\calL_{\bma,i-1}(\bmy_1),\  \calL_{\bma,i-1}(\bmy_2),\   \calL_{\bma,i-1}(\bmy_3)\Big)$\quad   for any $\bmy_1,\bmy_2,\bmy_3\in \R^{3^{i-1}}$.
		
		\item $\psi_{2,i}(\bmy_1,\bmy_2,\bmy_3)=\middleValue\Big(\psi_{2,i-1}(\bmy_1),\  \psi_{2,i-1}(\bmy_2),\   \psi_{2,i-1}(\bmy_3)\Big)$\quad   for any $\bmy_1,\bmy_2,\bmy_3\in \R^{3^{i-1}d}$.
	\end{itemize}
	See the illustrations in Figure~\ref{fig:phi123}.
	
	\begin{figure}[!htp]
		\centering
		\begin{subfigure}[c]{0.94\textwidth}
			\centering
			\includegraphics[width=0.99\textwidth]{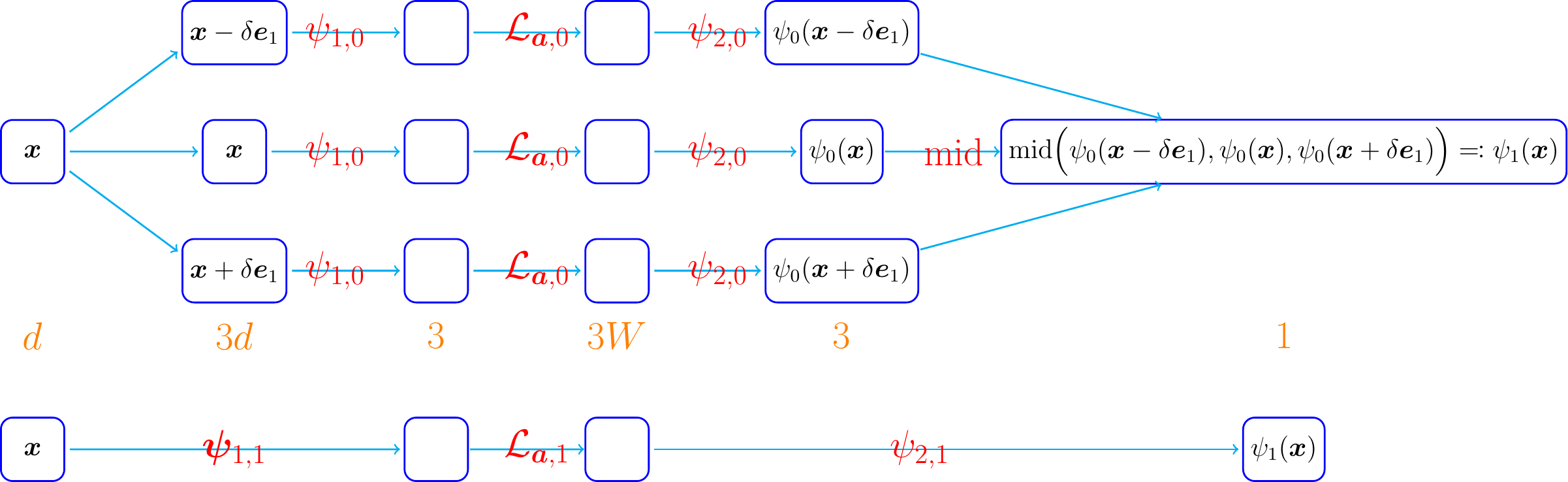}
			\subcaption{An illustration of the network architecture implementing $\psi_1=\psi_{2,1}\circ \calL_{\bma,1} \circ\bmpsi_{1,1}$ based on $\psi_0=\psi_{2,0}\circ \calL_{\bma,0} \circ\psi_{1,0}$. The top architecture is in detail, while the bottom one is just a sketch of the top one. The orange numbers indicate the number of neurons in each layer.}
		\end{subfigure}
		
		\vspace{25pt}
		\begin{subfigure}[c]{0.94\textwidth}
			\centering            \includegraphics[width=0.99\textwidth]{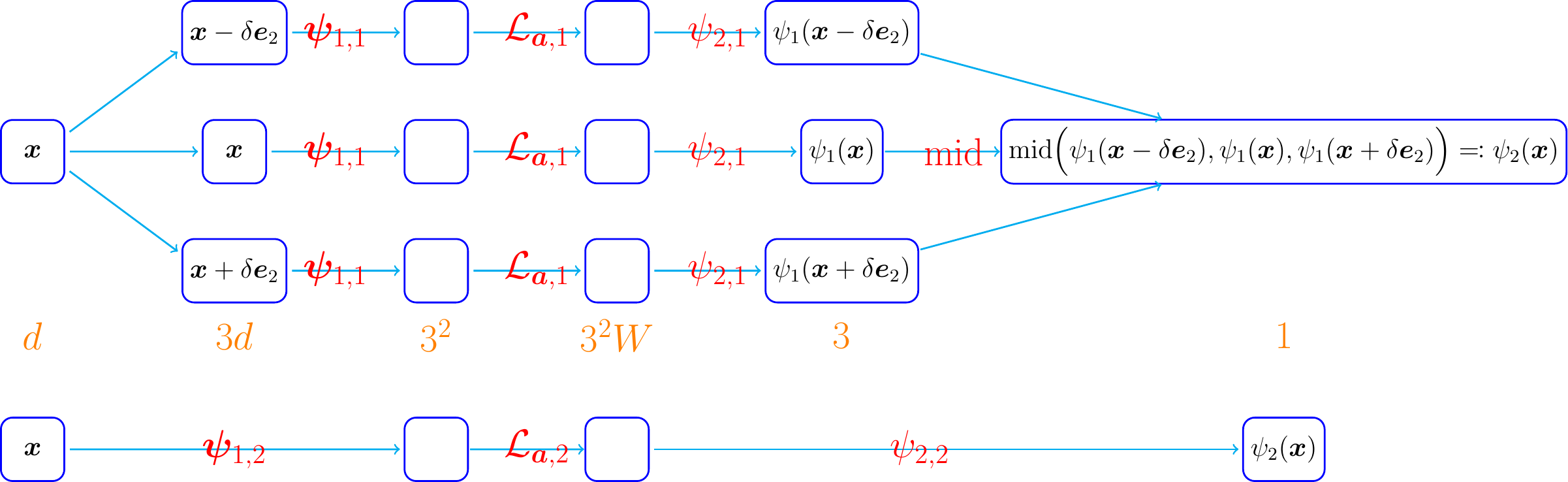}
			\subcaption{An illustration of the network architecture implementing $\psi_2=\psi_{2,2}\circ \calL_{\bma,2} \circ\bmpsi_{1,2}$ based on $\psi_1=\psi_{2,1}\circ \calL_{\bma,1} \circ\bmpsi_{1,1}$. The top architecture is in detail, while the bottom one is just a sketch of the top one. The orange numbers indicate the number of neurons in each layer.}
		\end{subfigure}
		
		\vspace{25pt}
		\begin{subfigure}[c]{0.94\textwidth}
			\centering           \includegraphics[width=0.99\textwidth]{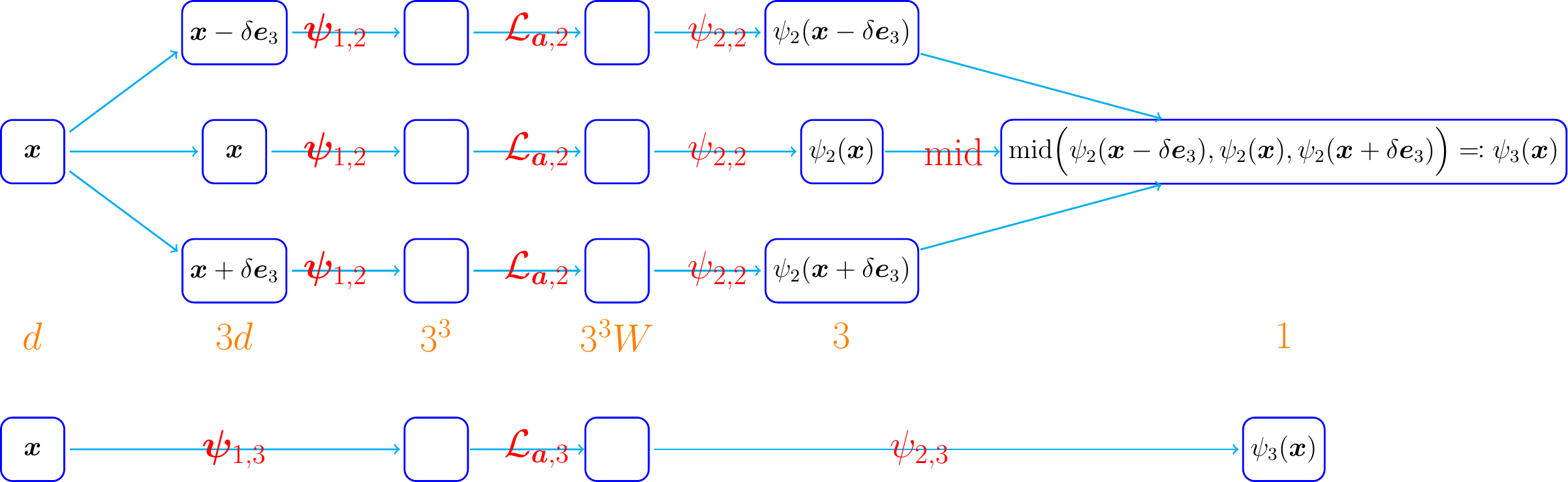}
			\subcaption{An illustration of the network architecture implementing $\psi_3=\psi_{2,3}\circ \calL_{\bma,3} \circ\bmpsi_{1,3}$ based on $\psi_2=\psi_{2,2}\circ \calL_{\bma,2} \circ\bmpsi_{1,2}$. The top architecture is in detail, while the bottom one is just a sketch of the top one. The orange numbers indicate the number of neurons in each layer.}
		\end{subfigure}
		\caption{Illustrations of the implementations of $\phi_1$, $\phi_2$, and $\phi_3$. The inductive implementations of $\phi_4,\cdots,\phi_d$ are similar.
		}
		\label{fig:phi123}
	\end{figure}

	By Lemma~\ref{lem:approx:trifling}, we have
	\begin{equation*}
		|\phi(\bmx)-\tildef(\bmx)|\le \varepsilon+d\cdot\omega_\tildef(\delta)\le \omega_\tildef(\sqrt{d}\, 2^{-n})+2^{-n+1}\quad \tn{for any $\bmx\in [0,1]^d$,}
	\end{equation*}
	where $\phi\coloneqq \psi_d =\psi_{2,d}\circ \calL_{\bma,d} \circ \bmpsi_{1,d}$. 
	By defining $\calL\coloneqq \calL_{\bma,d}$, $\bmphi_1 \coloneqq\bmpsi_{1,d}$, and $\phi_2 \coloneqq\psi_{2,d}$, we have 
	\begin{equation*}
		|\phi_2\circ\calL\circ\bmphi_1(\bmx)-\tildef(\bmx)|\le  \omega_\tildef(\sqrt{d}\, 2^{-n})+2^{-n+1}\quad \tn{for any $\bmx\in [0,1]^d$,}
	\end{equation*}	
	As shown in Figure~\ref{fig:phi123}, $\calL\coloneqq \calL_{\bma,d}$ is a linear map from $\R^{3^d}$ to $\R^{3^dn}$ determined by $\bma=(a_1,a_2,\cdots,a_n)\in [0,\tfrac{1}{3})^n$, which depends on $f$ and $n$. Moreover, as shown in Figure~\ref{fig:phi123}, $\bmphi_1 \coloneqq\bmpsi_{1,d}$ and $\phi_2 \coloneqq\psi_{2,d}$ are independent of $f$ and can be implemented by ReLU networks with 
	\[\le 3^d (2^{dn+4})+3d(d+1)(3^{d-1}+3^{d-2}+\cdots+3^0)\le 3^d 2^{dn+5}\] 
	and 
	\[\le 3^d(2^{dn+5}n)+280(3^{d-1}+3^{d-2}+\cdots+3^0)\le 3^d 2^{dn+8}n\] parameters, respectively.\footnote{As shown Lemma $3.1$ of \cite{shijun3}, ``$\middleValue(\cdot,\cdot,\cdot)$'' can be implemented by a ReLU network with width $14$ and depth $2$, which has $\le (3+1)\times 14+(14+1)\times 14+(14+1)= 280 $ parameters.}
	
	Note that $\omega_f(r)=s \cdot \omega_\tildef(r)$ for any $r\ge 0$. Therefore, we have
	\begin{equation*}
		\begin{split}
			&\quad \|s\cdot(\phi_2\circ \calL \circ \bmphi_1)+b -f\|_{L^\infty([0,1]^d)}
			=\Big\|s\cdot (\phi_2\circ \calL \circ \bmphi_1)+ b-\big(s\cdot \tildef+b\big)\Big\|_{L^\infty([0,1]^d)}\\
			&=s\|\phi_2\circ\calL\circ\bmphi_1-\tildef\|_{L^\infty([0,1]^d)}\le s\cdot \omega_\tildef(\sqrt{d}\, 2^{-n})+ 2^{-n+1}  s
			= \omega_f(\sqrt{d}\,2^{-n})+2^{-n+2}\omega_f(\sqrt{d}).
		\end{split}
	\end{equation*}
	So we finish the proof.
\end{proof}

\subsection{Proof of Theorem~\ref{thm:main:three:parameters}}

To simplify the proof of  Theorem~\ref{thm:main:three:parameters}, we introduce two lemmas below.
First, we need to establish a lemma showing how to store many parameters in one intrinsic parameter via a fixed network.
\begin{lemma}\label{lem:store:parameter}
	Given any $m,n\in\N$, there exists a vector-valued function $\bmphi:\R\to\R^n$ realized by a ReLU network  such that: For any $a_i\in [0,1)$ with $2^m a_i\in \N$ for $i=0,1,\cdots,n$, there exists a real number $v\in [0,1)$ such that
	\begin{equation*}
		\bmphi(v)=(a_1,a_2,\cdots,a_n).
	\end{equation*}
\end{lemma}
Next, we establish another lemma using a ReLU network to uniformly approximate multiplication operation $\psi(x,y)=xy$ well. 
\begin{lemma}\label{lem:xy}
	For any $M>0$ and $\eta>0$, there exists a function $\psi_{\eta}:\R^2\to\R$ realized by a ReLU network  such that
	\begin{equation*}
		\psi_{\eta}(x,y)\rightrightarrows \psi(x,y)=xy\quad \tn{on $[-M,M]^2$\quad as $\eta\to 0^+$},
	\end{equation*}
	where $\rightrightarrows$ denotes the uniform convergence.
\end{lemma}

The proof of Lemma~\ref{lem:store:parameter} is placed later in this section. Lemma~\ref{lem:xy} is just a direct result of Lemma~$4.2$ of \cite{shijun3}.
With Lemmas~\ref{lem:store:parameter} and \ref{lem:xy} in hand, we are ready to  prove Theorem~\ref{thm:main:three:parameters}.
\begin{proof}[Proof of Theorem~\ref{thm:main:three:parameters}]
	For any $\varepsilon>0$, choose a large $n=n(\varepsilon,\alpha,\lambda)\in \N^+$ such that
	\begin{equation*}
		5\lambda d^{\alpha/2} 2^{-\alpha n}\le \varepsilon/2.
	\end{equation*}
	Since $f$ is a H\"older continuous function on $[0,1]^d$ of order $\alpha\in (0,1]$ with a H\"older constant $\lambda>0$, we have $\omega_f(r)\le \lambda r^{\alpha}$ for any $r\ge 0$.
	By Theorem~\ref{thm:mainInfty}, 
	there exist two functions $\bmphi_1:\R^d\to \R^{3^d}$ and $\phi_2:\R^{3^dn}\to \R$, implemented by $f$-independent ReLU networks, such that
	\begin{equation}\label{eq:error1}
		\|s(\phi_2\circ \calL \circ\bmphi_1)+b-f\|_{L^\infty([0,1]^d)}\le \omega_f(\sqrt{d}\, 2^{-n})+2^{-n+2}\omega_f(\sqrt{d})
		\le 5\lambda d^{\alpha/2} 2^{-\alpha n}\le \varepsilon/2,
	\end{equation}
	where $s=2\omega_f(\sqrt{d})\le 2\lambda d^{\alpha/2}$, $b=f(\bmzero)-\omega_f(\sqrt{d})$, and $\calL:\R^{3^d}\to \R^{3^dn}$ is a linear map given by \[\calL(y_1,\cdots,y_{3^d})=\Big(\calL_0(y_1),\cdots,\calL_0(y_{3^d})\Big)\quad \tn{ for any $\bmy=(y_1,\cdots,y_{3^d})\in \R^{3^d}$,}\]
	where $\calL_0:\R\to \R^n$ is a linear map given by $\calL_0(t)=(a_1t, a_2 t,\cdots,a_n t)$ with $a_1,a_2,\cdots,a_n\in [0,\tfrac{1}{3})$ determined by $f$ and $n$. 
	Since $a_1,a_2,\cdots,a_n$ are repeated $3^d$ times in the definition of $\calL$, there are $3^d n$ parameters in total. We will show how to store these $3^dn$ parameters in one intrinsic parameter $v$ via an $f$-independent ReLU network as shown in the following two steps.
	\begin{itemize}
		\item Regard $a_1,a_2,\cdots,a_n$ as inputs, but not parameters. See the difference in Figure~\ref{fig:ai:input:parameter}. 
		Then, we only need to store $a_1,a_2,\cdots,a_n$ one time by  copying them $3^d$ times. 
		\item  As stated in the proof of Theorem~\ref{thm:mainInfty}, $a_1,a_2,\cdots,a_n$ have finite binary representations. Then, we can store them in a key parameter $v$ and use  an $f$-independent ReLU network to extract them from $v$.
	\end{itemize}
	The details of these two steps can be found below.
	\mystep{1}{Regard $a_1,a_2,\cdots,a_n$ as inputs and copy them $3^d$ times.  }
	
	Since  $a_1,a_2,\cdots,a_n$ are regraded as inputs,
	the implementation of $\calL_0(t)=(a_1t, a_2 t,\cdots,a_n t)$ requires multiplication operations. This means that 
	we need to  approximate $\psi(x,y)=xy$ well via an $f$-independent ReLU network. See Figure~\ref{fig:ai:input:parameter} for illustrations.

	Denote $\bma=(a_1,a_2,\cdots,a_n)$ and \begin{equation*}
		\bmphi_1(\bmx)=\Big(\phi_{1,1}(\bmx),\,\phi_{1,2}(\bmx),\,\cdots,\, \phi_{1,3^d}(\bmx)\Big) \quad \tn{for any $\bmx\in [0,1]^d.$}
	\end{equation*}
	Then define 
	\begin{equation*}
		M\coloneqq 1+\sup\Big\{|\phi_{1,j}(\bmx)|:\bmx\in [0,1]^d,\, j=1,2,\cdots,3^d\Big\}.
	\end{equation*}
	By Lemma~\ref{lem:xy}, 
	there exists a function $\psi_{\eta}:\R^2\to\R$ realized by a ReLU network such that
	\begin{equation*}
		\psi_{\eta}(x,y)\rightrightarrows \psi(x,y)=xy\quad \tn{on $[-M,M]^2$\quad as $\eta\to 0^+$}.
	\end{equation*}
	
	\begin{figure}[ht!]        
		\centering
		\begin{subfigure}[b]{0.32\textwidth}
			\centering            \includegraphics[width=0.75\textwidth]{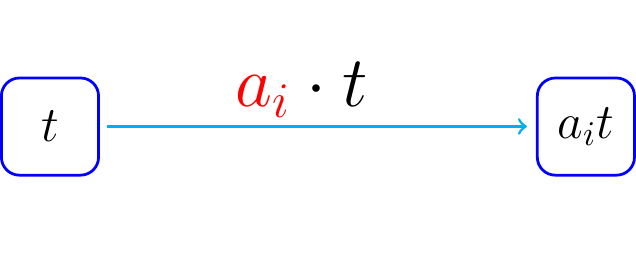}
			\subcaption{}
		\end{subfigure}
		\begin{subfigure}[b]{0.448\textwidth}
			\centering            \includegraphics[width=0.75\textwidth]{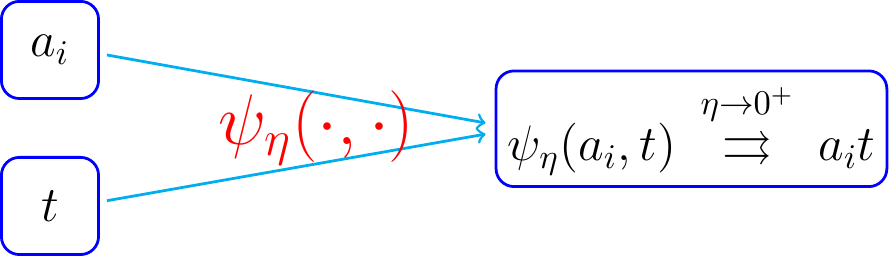}
			\subcaption{}
		\end{subfigure}
		\caption{Illustrations of two methods getting/approximating $a_i t$ for $i=1,2,\cdots,n$ and $t\in \{\phi_{1,j}(\bmx): \bmx\in [0,1]^d,\   j=1,2,\cdots,3^d\}$.
			(a) By regarding $a_i$ as a parameter, one can easily get the product of an input $t$ and an parameter $a_i$. (b) By regarding $a_i$ as an input, one needs to use a ReLU network to approximate the multiplication operation for approximating $a_i t$ well.}
		\label{fig:ai:input:parameter}
	\end{figure}
	
	Note that $\calL$ and $\calL_0$ depend on $\bma=(a_1,a_2,\cdots,a_n)$. For clarity, we denote $\calL_\bma=\calL$ and $\calL_{\bma,0}=\calL_0$.
	For any  $a_1,a_2,\cdots,a_n\in [0,\tfrac{1}{3})\subseteq [-M,M]$ and \begin{equation*}
		t\in \Big\{\phi_{1,j}(\bmx): \bmx\in [0,1]^d,\   j=1,2,\cdots,3^d\Big\}\subseteq [-M,M],
	\end{equation*}
	we can use 
	\begin{equation*}
		\calL_{\bma,0,\eta}(t)\coloneqq\Big(\psi_{\eta}(a_1,t),\, \psi_{\eta}(a_2, t),\,\cdots,\,\psi_{\eta}(a_n, t)\Big)
	\end{equation*}
	to approximate 
	\begin{equation*}
		\calL_{\bma,0}(t)=(a_1t, a_2 t,\cdots,a_n t)=\Big(\psi(a_1,t),\, \psi(a_2, t),\,\cdots,\,\psi(a_n, t)\Big).
	\end{equation*}
	Then 
	\begin{equation*}
		\calL_{\bma,\eta}(y_1,\cdots,y_{3^d})\coloneqq \Big(\calL_{\bma,0,\eta}(y_1),\ \cdots,\ \calL_{\bma,0,\eta}(y_{3^d})\Big)
	\end{equation*}
	can also approximate 
	\begin{equation*}
		\calL_{\bma}(y_1,\cdots,y_{3^d})= \Big(\calL_{\bma,0}(y_1),\  \cdots,\ \calL_{\bma,0}(y_{3^d})\Big)
	\end{equation*}
	well. Define $\tildephi_\eta:\R^{n+3^d}\to\R^{3^dn}$ and $\tildephi:\R^{n+3^d}\to\R^{3^dn}$ via
	\begin{equation*}
		\tildephi_\eta(\bmy,\bma)\coloneqq \calL_{\bma,\eta}(y_1,\cdots,y_{3^d})\quad \tn{for any $\bma\in [0,\tfrac13)^n$ and $\bmy=(y_1,\cdots,y_{3^d}) \in [-M,M]^{3^d}$}
	\end{equation*}
	and 
	\begin{equation*}
		\tildephi(\bmy,\bma)\coloneqq \calL_{\bma}(y_1,\cdots,y_{3^d})\quad \tn{for any $\bma\in [0,\tfrac13)^n$ and $\bmy=(y_1,\cdots,y_{3^d})\in [-M,M]^{3^d}$}.
	\end{equation*}

	Since 
	\begin{equation*}
		\psi_{\eta}(x,y)\rightrightarrows \psi(x,y)=xy\quad \tn{on $[-M,M]^2$\quad as $\eta\to 0^+$},
	\end{equation*}
	it is easy to verify that
	\begin{equation*}
		\tildephi_{\eta}(\bmy,\bma)\rightrightarrows \tildephi(\bmy,\bma)\quad \tn{for $\bmy\in [-M,M]^{3^d}$ and $\bma\in [0,\tfrac13)^n$\quad as $\eta\to 0^+$}.
	\end{equation*}
	Note that $\bmphi_1(\bmx)\in [-M,M]^{3^d}$ for any $\bmx\in [0,1]^d$. Then, 
	\begin{equation*}
		\phi_2\circ\tildephi_{\eta}\Big(\bmphi_1(\bmx),\bma\Big)\rightrightarrows \phi_2\circ\tildephi\Big(\bmphi_1(\bmx),\bma\Big)\quad \tn{for $\bmx\in [0,1]^d$ and $\bma\in [0,\tfrac13)^n$\quad as $\eta\to 0^+$}.
	\end{equation*}
	The fact $\tildephi(\bmy,\bma)=\calL_\bma(\bmy)=\calL(\bmy)$
	implies $\phi_2\circ\tildephi\Big(\bmphi_1(\bmx),\bma\Big)=\phi_2\circ\calL\circ\bmphi_1(\bmx)$. 
	Therefore, 
	\begin{equation*}
		\phi_2\circ\tildephi_{\eta}\Big(\bmphi_1(\bmx),\bma\Big)\rightrightarrows
		\phi_2\circ\calL\circ\bmphi_1(\bmx)\quad \tn{for $\bmx\in [0,1]^d$ and $\bma\in [0,\tfrac13)^n$ \quad as $\eta\to 0^+$}.
	\end{equation*}
	Choose a small $\eta=\eta(n)>0$ such that
	\begin{equation}\label{eq:error2}
		\Big| \phi_2\circ\tildephi_{\eta}\Big(\bmphi_1(\bmx),\bma\Big)
		-\phi_2\circ\calL\circ\bmphi_1(\bmx)\Big|
		\le 2^{-n}\quad \tn{for any $\bmx\in [0,1]^d$ and $\bma\in [0,\tfrac13)^n$ }.
	\end{equation}
	
	Recall that $\psi_\eta$ can be realized by an $f$-independent ReLU network. It is easy to verify that $\tildephi_\eta$ can also be realized by an $f$-independent ReLU network.

	\mystep{2}{Store $a_1,a_2,\cdots,a_n$ in a key parameter $v$.}
	
	As we can see from the proof of Theorem~\ref{thm:mainInfty}, $a_i\in [0,\tfrac{1}{3})$ with $2^m a_i\in \N$ for $i=1,2,\cdots,n$, where $m=2^{dn+1}$. That is,
	\begin{equation*}
		a_i\in \Big\{\bin 0.\theta_1\cdots\theta_{m}:\  \theta_\ell\in \{0,1\},\  \ell=1,2,\cdots,m \Big\}.
	\end{equation*}
	Then, by Lemma~\ref{lem:store:parameter}, there exists a real number $v\in [0,1)$ and a vector function $\bmphi_0:\R\to\R^n$ implemented by a ReLU network independent of $a_1,a_2,\cdots,a_n$ such that
	\begin{equation*}
		\bmphi_0(v)=(a_1,a_2,\cdots,a_n)=\bma.
	\end{equation*}
	Next, we can define the final network-generated function $\phi:\R^{d+1}\to\R$ by
	\begin{equation*}
		\phi(\bmx,v)\coloneqq \phi_2\circ \tildephi_\eta\Big(\bmphi_1(\bmx),\bmphi_0(v)\Big)
		=\phi_2\circ \tildephi_\eta\Big(\bmphi_1(\bmx),\bma\Big)
	\end{equation*}
	for any $\bmx\in [0,1]^d$ and $v\in [0,1)$.
	
	Since the ReLU network realizing $\bmphi_0$ is independent of $a_1,a_2,\cdots,a_n$, and hence independent of $f$. Recall that $\bmphi_1$, $\phi_2$, and $\tildephi_\eta$ can be implemented by $f$-independent ReLU networks. Hence,
	\begin{equation*}
		\phi(\bmx,v)= \phi_2\circ \tildephi_\eta\Big(\bmphi_1(\bmx),\bmphi_0(v)\Big)
	\end{equation*}
	can also implemented by an $f$-independent ReLU network.
	It remains to estimate the error. By Equations~\eqref{eq:error1} and \eqref{eq:error2}, we have
	\begin{equation*}
		\begin{split}
			\big|s\phi(\bmx,v)+b-f(\bmx)\big|
			&\le 
			\Big|s\phi(\bmx,v)+b-\big(s\phi_2\circ\calL\circ\bmphi_1(\bmx)+b\big)\Big|  +   \Big|s\phi_2\circ\calL\circ\bmphi_1(\bmx)+b-f(\bmx)\Big| \\
			&\le s\big|\phi(\bmx,v)-\phi_2\circ\calL\circ\bmphi_1(\bmx)\big|  + \varepsilon/2\\
			&\le 2\lambda d^{\alpha/2}\Big|\phi_2\circ \tildephi_\eta\Big(\bmphi_1(\bmx),\bma\Big)-\phi_2\circ\calL\circ\bmphi_1(\bmx)\Big|  + \varepsilon/2\\
			&\le 2\lambda d^{\alpha/2} 2^{-n} +\varepsilon/2\le
			5\lambda d^{\alpha/2} 2^{-\alpha  n} +\varepsilon/2\le\varepsilon/2 +\varepsilon/2=\varepsilon.
		\end{split}
	\end{equation*}
	So we finish the proof.
\end{proof}

Finally, let us prove Lemma~\ref{lem:store:parameter} to end this section.
\begin{proof}[Proof of Lemma~\ref{lem:store:parameter}]
	Since $a_i\in [0,1)$ with $2^m a_i\in \N$ for $i=1,2,\cdots,n$,  $a_i$ can be represented as a binary form 
	\begin{equation*}
		a_i=\bin 0. a_{i,1}a_{i,2}\cdots  a_{i,m}.
	\end{equation*}
	Denote
	\begin{equation*}
		v=\sum_{i=1}^n 2^{-m(i-1)} a_i
		=\bin 0.
		\underbrace{  a_{1,1}\cdots a_{1,m} }_{\tn{store }\  a_1}
		\underbrace{  a_{2,1}\cdots a_{2,m} }_{\tn{store }\  a_2}
		\   \cdots\ 
		\underbrace{  a_{n,1}\cdots a_{n,m} }_{\tn{store }\ a_n},
	\end{equation*}
	which requires pretty high precision.
	It is easy to extract $a_i$ from $v$ via the floor function ($\lfloor\cdot\rfloor$), i.e.,
	\begin{equation*}
		a_i=\lfloor 2^{mi}v\rfloor / 2^m - \lfloor 2^{m(i-1)}v\rfloor\quad \tn{for $i=1,2,\cdots,n$.}
	\end{equation*}
	Next, we need to use a ReLU network to replace the floor function. Let $g:\R\to\R$ be the continuous piecewise linear function with the following breakpoints:
	\begin{equation*}
		(\ell, \ell)\quad \tn{and}\quad  (\ell+1-\delta,\ell)\quad \tn{for $\ell=0,1,\cdots,2^{mn}-1$,\quad where $\delta=2^{-mn}$.}
	\end{equation*}
	Clearly, $g$ can be realized by a ReLU network independent of $a_1,a_2,\cdots,a_n$ and 
	\begin{equation*}
	    g(x)=\lfloor x\rfloor \quad \tn{for any $\displaystyle  x\in \bigcup_{\ell=0}^{2^{mn}-1} \big[\ell,\, \ell+1-\delta\big].$}
	\end{equation*}
	Note that $v\in  \bigcup_{\ell=0}^{2^{mn}-1} \big[\ell,\, \ell+1-\delta\big].$ By defining 
	\begin{equation*}
		g_i(t)\coloneqq  g( 2^{mi}t) / 2^m - g( 2^{m(i-1)} t)\quad \tn{for $i=1,2,\cdots,n$ and any $t\in\R$,}
	\end{equation*}
	we have
	\begin{equation*}
		a_i=g( 2^{mi}v) / 2^m - g( 2^{m(i-1)} v)=g_i(v)\quad \tn{for $i=1,2,\cdots,n$.}
	\end{equation*}
	Next, The target function $\bmphi$ can be defined via
	\begin{equation*}
		\bmphi(t)=\Big(g_1(t),\, g_2(t),\,\cdots,\,g_n(t)\Big)\quad \tn{for any $t\in\R$.}
	\end{equation*}
	Thus, we have 
	\begin{equation*}
		\bmphi(v)=\Big(g_1(v),\, g_2(v),\,\cdots,\,g_n(v)\Big)=(a_1,a_2,\cdots,a_n).
	\end{equation*}
	and $\bmphi$ can be realized by a ReLU network independent of $a_1,a_2,\cdots,a_n$. So we finish the proof.
\end{proof}
\section{Proof of Theorem~\ref{thm:mainOld}} 
\label{sec:proofMainOld}
In this section, we first present the proof sketch of  Theorem~\ref{thm:mainOld} in Section~\ref{sec:proof:sketch:main:old}, and then give the detailed proof in Section~\ref{sec:proof:main:old} based on Proposition~\ref{prop:phi4j}, which will be proved later in Section~\ref{sec:proof:prop}.

\subsection{Sketch of proof}
\label{sec:proof:sketch:main:old}
Before proving Theorem~\ref{thm:mainOld}, let us present the key steps as follows.
\begin{enumerate}[1.]
	\item Set $K=2^n$, divide $[0,1]^d$ into $K^d$ cubes $Q_\bmbeta$ for $\bmbeta\in\{0,1,\cdots,K-1\}^d$ and the trifling region $\Omega([0,1]^d,K,\delta)$, and denote $\bmx_\bmbeta$ as the vertex of $Q_\bmbeta$ with minimum $\|\cdot\|_1$ norm for each $\bmbeta$. See Figure~\ref{fig:Qbeta+xbeta} for illustrations.
	
	\item Design a ReLU sub-network to implement a function $\phi_1:\R^d\to \R$, independent of $f$, projecting the whole $Q_\bmbeta$ to a number determined by $\bmbeta$ in $\{4^j:j=1,2,\cdots,K^d\}$ for each $\bmbeta$. 
	
	\item Design a linear map $\calL:\R\to \R^n$, given by $\calL(t)=(a_1t, a_2 t, \cdots , a_n t)$, for later use, where $a_1, a_2, \cdots, a_n$ are determined by $f$ and $n$.
	
	\item Design a ReLU sub-network to implement a function $\phi_2:\R^n\to \R$, independent of $f$, such that $\phi_2\circ\calL\circ\phi_1(\bmx)\approx f(\bmx_\bmbeta)$ for any $\bmx \in Q_\bmbeta$ and each $\bmbeta\in\{0,1,\cdots,K-1\}^d$. Then $\phi_2\circ\calL\circ\phi_1$ approximates $f$ well outside of $\Omega([0,1]^d,K,\delta)$. 
	
	\item Estimate the approximation error of $\phi_2\circ\calL\circ\phi_1\approx f$.
\end{enumerate}
\begin{figure}[!htp]
	\centering
	\begin{subfigure}[b]{0.313\textwidth}
		\centering            \includegraphics[width=0.75\textwidth]{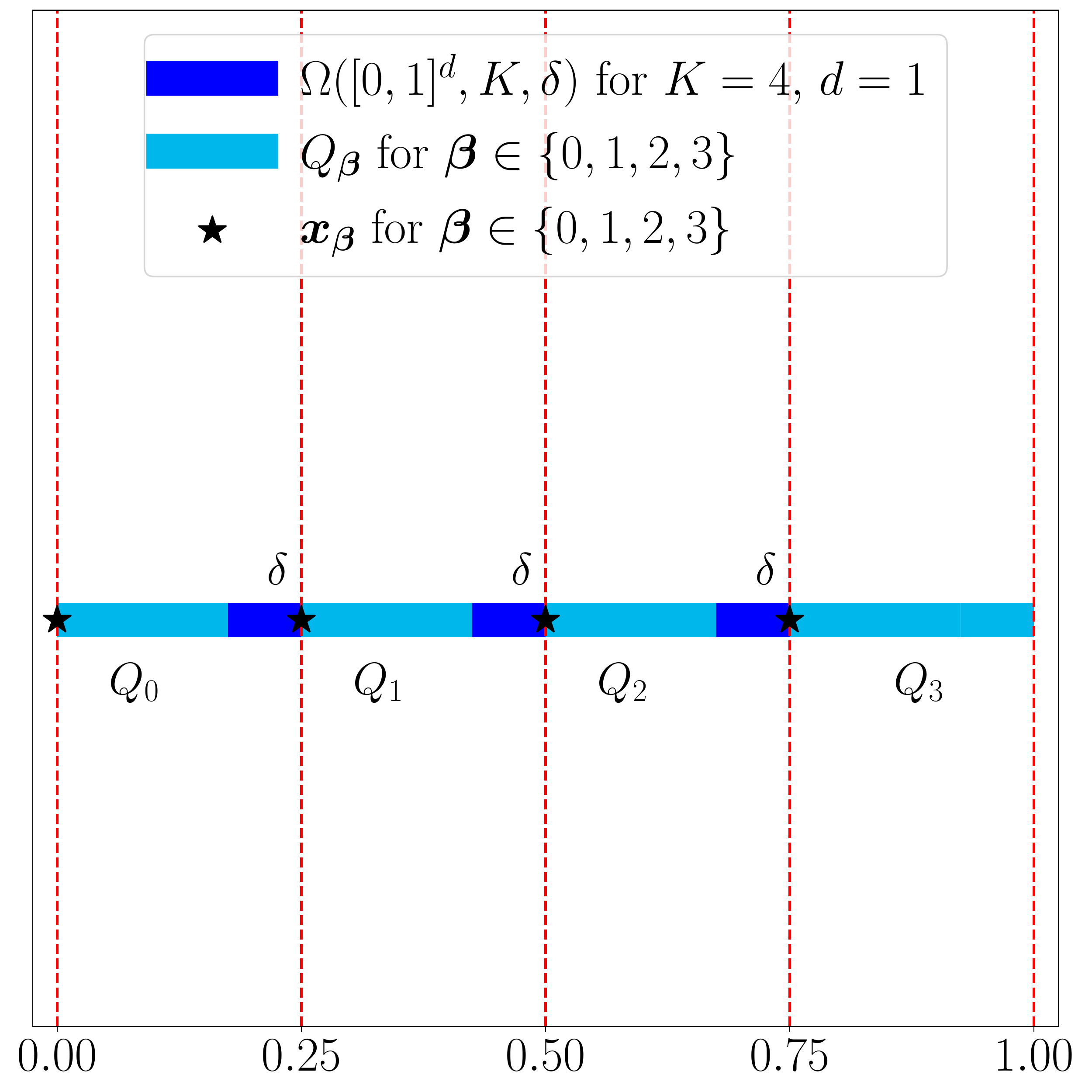}
		\subcaption{}
	\end{subfigure}
	\begin{subfigure}[b]{0.313\textwidth}
		\centering            \includegraphics[width=0.75\textwidth]{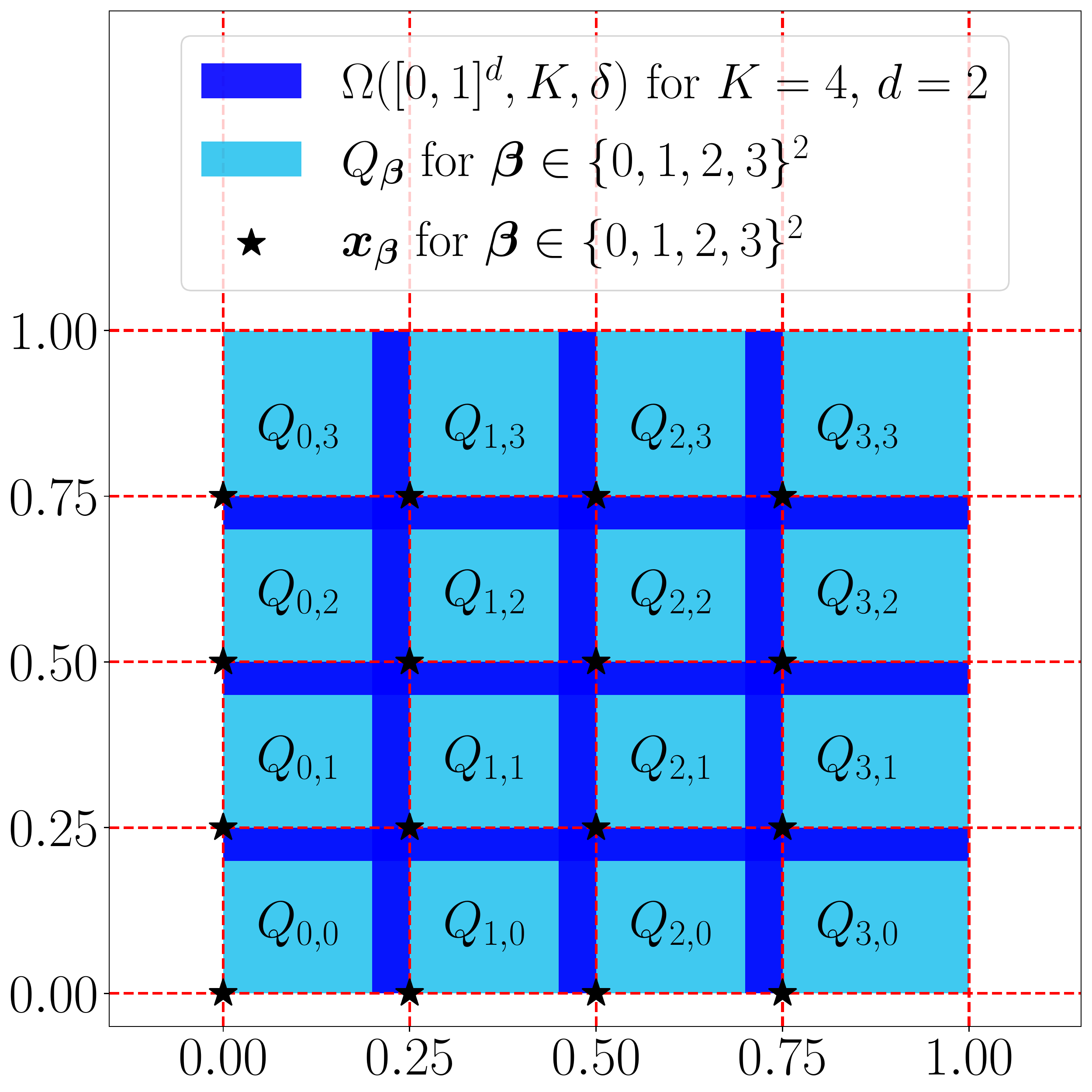}
		\subcaption{}
	\end{subfigure}
	\caption{Illustrations of $\Omega([0,1]^d,K,\delta)$, $Q_\bmbeta$ and $\bmx_\bmbeta$ for $\bmbeta\in \{0,1,\cdots,K-1\}^d$. (a) $K=4,\ d=1$. (b)  $K=4,\ d=2$.}
	\label{fig:Qbeta+xbeta}
\end{figure}

As we shall see later, the constructions of $\phi_1$ and $\calL$ are not difficult. The most technical part is to 
design $\phi_2$ implemented by a ReLU network, which relies on the following proposition.
\begin{proposition}
	\label{prop:phi4j}
	Given any $J\in \N^+$, there exists a function $\phi$ implemented by a ReLU network with width $2$ and depth $2J+2$  such that: For any $\theta_1,\theta_2,\cdots,\theta_J\in \{0,1\}$, we have $\phi(x)\in [0,1]$ for any $x\in \R$ and
	\begin{equation}
		\label{eq:phi4j}
		\phi(4^ja)=\theta_j\quad \tn{for $j=1,2,\cdots,J$},\quad \tn{where}\ a=\sum_{j=1}^J \theta_j4^{-j}.
	\end{equation}
\end{proposition}
The proof of this proposition can be found in Section~\ref{sec:proof:prop}. We shall point out that the function $\phi$ in this proposition is independent of $\theta_1,\theta_2,\cdots,\theta_J\in\{0,1\}$. 

\subsection{Constructive proof}\label{sec:proof:main:old}
Now we are ready to give the detailed proof of Theorem~\ref{thm:mainOld}.
\begin{proof}[Proof of Theorem~\ref{thm:mainOld}]
	The proof consists of five steps.
	\mystep{1}{Set up.}
	Set $K=2^n $  and let $\delta>0$ be a small number determined later. Then define $\bmx_\bmbeta\coloneqq \bmbeta/K$ and divide $[0,1]^d$ into $K^d$  cubes $Q_\bmbeta$ for $\bmbeta\in\{0,1,\cdots,K-1\}^d$ and a small region $\Omega([0,1]^d,K,\delta)$. Namely,
	\begin{equation*}
		Q_\bmbeta\coloneqq \Big\{\bmx=(x_1,x_2,\cdots,x_d): x_i\in [\tfrac{\beta_i}{K},\tfrac{\beta_i+1}{K}-\delta]\ \tn{for}\ i=1,\cdots,d\Big\},
	\end{equation*}
	for $\bmbeta=(\beta_1,\beta_2,\cdots,\beta_d)\in \{0,1,\cdots,K-1\}^d$. 
	Clearly, $\Omega([0,1]^d,K,\delta)= [0,1]^d\backslash (\bigcup_{\bmbeta\in \{0,1,\cdots,K-1\}^d}Q_\bmbeta)$. See  Figure~\ref{fig:Qbeta+xbeta} for illustrations.
	
	\mystep{2}{Construct $\phi_1$.}
	Let $g$ be a ``step function'' such that
	\begin{itemize}
		\item $g(\tfrac{k}{K})=g(\frac{k+1}{K}-\delta)=k$ for $k=0,1,\cdots,K-1$ and $g(1)=K-1$.
		\item $g$ is linear between any two adjacent points of \[\{\tfrac{k}{K}:k=0,1,\cdots,K\}\cup \{\tfrac{k+1}{K}-\delta:k=0,1,\cdots,K-1\}.\]
	\end{itemize}
	Then, for any $\bmx=(x_1,\cdots,x_d)\in Q_\bmbeta$ and $\bmbeta=(\beta_1,\cdots,\beta_d)\in \{0,1,\cdots,K-1\}^d$, we have
	\[g(x_i)=\beta_i\quad \tn{for $i=1,2,\cdots,d$.} \]
	Also, such a function $g$ can be easily realized by a one-hidden-layer ReLU network with width $2K$.

	Let $h$ be a function satisfying $h(j)=4^j$ for $j=1,2,\cdots,K^d$. Such a function $h$ can be easily realized by a one-hidden-layer ReLU network with width $K^d$. Then the desired function  $\phi_1:\R^d\to\R$  can be defined via
	\begin{equation*}
		\phi_1(\bmx)\coloneqq h\Big(1+\sum_{i=1}^d g(x_i)K^{i-1}\Big)=h\circ\varrho\Big(g(x_1),\cdots,g(x_d)\Big)\quad \tn{for any $\bmx=(x_1,\cdots,x_d)\in \R^d$,}
	\end{equation*}
	where $\varrho:\R^d\to\R$ is a linear function defined by 
	\begin{equation*}
		\varrho(\bmy)=1+\sum_{i=1}^d y_i K^{i-1}\quad \tn{for any $\bmy=(y_1,\cdots,y_d)\in \R^d$.}
	\end{equation*}  
	Clearly, $\varrho$ is a bijection (one-to-one map) from $\bmbeta\in \{0,1,\cdots,K-1\}^d$ to $\varrho(\bmbeta)=1+\sum_{i=1}^d \beta_i K^{i-1}\in \{1,2,\cdots,K^d\}$.
	
	Then, for any $\bmx\in Q_\bmbeta$ and $\bmbeta\in\{0,1,\cdots,K-1\}^d$, we have
	\begin{equation}
		\label{eq:psi1}
		\phi_1(\bmx)=h\circ\varrho\Big(g(x_1),\cdots,g(x_d)\Big)=h\circ\varrho\big(\beta_1,\cdots,\beta_d\big)=4^{\varrho(\bmbeta)}=4^{1+\sum_{i=1}^d \beta_i K^{i-1}}.
	\end{equation}
	Apparently, $\phi_1$ is independent of $f$ and it can be realized by a ReLU network with \[\le d(2K\times 3+1) \ +\  (K+1)\  +\   3K^d+1\le 11K^d +2=11\times 2^{dn}+2\le 2^{dn+4}\] parameters. 
	
	\mystep{3}{Construct $\calL$.}
	For each $\bmbeta\in \{0,1,\cdots,K-1\}^d$, it follows from $f(\bmx_\bmbeta)\in [0,1]$ that there exist 
	$\xi_{\bmbeta,1},\cdots,\xi_{\bmbeta,n}$ such that
	\begin{equation}
		\label{eq:fapproxbin}
		|f(\bmx_\bmbeta)-\bin 0.\xi_{\bmbeta,1}\cdots\xi_{\bmbeta,n}|\le 2^{-n}.
	\end{equation}
	Given any $j\in \{1,2,\cdots,K^d\}$, there exists a unique $\bmbeta\in \{0,1,\cdots,K-1\}^d$ such that $j=1+\sum_{i=1}^d \beta_i K^{i-1}=\varrho(\bmbeta)$. Thus, for any $\ell\in\{1,2,\cdots,n\}$, we can define
	\begin{equation}
		\label{eq:theta=xi}
		\theta_{j,\ell}\coloneqq\xi_{\bmbeta,\ell}, \quad \tn{for $j=\varrho(\bmbeta)$ and $\bmbeta\in \{0,1,\cdots,K-1\}^d$.}
	\end{equation}
	
	Then the desired linear map $\calL$ can be defined via
	\begin{equation*}
		\calL(t)\coloneqq (a_1 t,a_2 t,\cdots, a_n t)\quad \tn{for any $t\in \R$,}
	\end{equation*}
	where $a_\ell=\sum_{j=1}^{K^d} \theta_{j,\ell}4^{-j}$ for $\ell=1,2,\cdots,n$. Clearly, for $\ell=1,2,\cdots,n$, we have
	\begin{equation*}
		a_\ell=\sum_{j=1}^{K^d} \theta_{j,\ell}4^{-j}\in [0,\tfrac{1}{3})
	\end{equation*}
	and 
	\begin{equation*}
		2^m a_\ell=2^{2^{dn+1}}a_\ell
		=4^{2^{dn}}a_\ell
		=4^{K^d}a_\ell\in \N, 
		\quad 
		\tn{where $m=2^{dn+1}$.}
	\end{equation*}

	\mystep{4}{Construct $\phi_2$.}
	Fix $\ell\in \{1,2,\cdots,n\}$, by Proposition~\ref{prop:phi4j} (set $J=K^d$ and $\theta_j=\theta_{j,\ell}$ therein), 
	there exists a function $\phi_{2,\ell}$ implemented by a ReLU network with width $2$ and depth $2K^d+2$ such that $\phi_{2,\ell}(t)\in [0,1]$ for any $t\in \R$ and
	\begin{equation}
		\label{eq:psi2w4j}
		\phi_{2,\ell}(4^ja_\ell)=\theta_{j,\ell}\quad \tn{for $j=1,2,\cdots,K^d$},\quad \tn{where}\ a_\ell=\sum_{j=1}^{K^d} \theta_{j,\ell}4^{-j}.
	\end{equation}
	Note that $\phi_{2,\ell}$ is independent of $\theta_{j,\ell}$ for $j=1,2,\cdots,K^d$, so it is also independent of $f$. Then the desired function $\phi_2:\R^n\to \R$ can be defined via
	\begin{equation*}
		\phi_2(\bmy)\coloneqq \sum_{\ell=1}^{n}2^{-\ell}\phi_{2,\ell}(y_\ell)\quad \tn{for any $\bmy=(y_1,\cdots,y_n)\in \R^n$.}
	\end{equation*}
	Then $\phi_{2}(\bmy)\in [0,1]$ for any $\bmy\in \R^n$ and
	$\phi_{2}$ can be implemented by a ReLU network, independent of $f$,  with \[\le \Big(6(2K^n+2)+(2+1)\Big)n+n+1=6n(2^{nd+1}+2)+4n+1\le 2^{nd+5}n \] parameters.
	
	\mystep{5}{Estimate the approximation error.}
	It  remains to estimate the approximation error. By Equations~\eqref{eq:psi1}, \eqref{eq:theta=xi}, and \eqref{eq:psi2w4j}, for  $\ell\in \{1,2,\cdots,n\}$, $\bmx\in Q_\bmbeta$, $j=\varrho(\bmbeta)=1+\sum_{i=1}^d \beta_i K^{i-1}\in \{1,2,\cdots,K^d\}$, and $\bmbeta\in \{0,1,\cdots,K-1\}^d$, we have
	\begin{equation*}
		\begin{split}
			\phi_2\circ \calL\circ \phi_1(\bmx)&=\phi_2\circ \calL\circ h\circ \varrho\Big(g(x_1),\cdots,g(x_d)\Big)=\phi_2\circ \calL\circ h\circ \varrho(\bmbeta)\\
			&= \phi_2\circ \calL(4^{\varrho(\bmbeta)})
			=\phi_2\circ \calL(4^j)
			= \phi_2(4^ja_1,4^ja_2,\cdots,4^ja_n)\\
			&=\sum_{\ell=1}^{n}2^{-\ell}\phi_{2,\ell}(4^ja_\ell)
			= \sum_{\ell=1}^{n}2^{-\ell}\theta_{j,\ell}
			=\sum_{\ell=1}^{n}2^{-\ell}\xi_{\bmbeta,\ell}.
		\end{split}
	\end{equation*}
	
	Then by Equation~\eqref{eq:fapproxbin}, for any $\bmx\in Q_\bmbeta$ and $\bmbeta\in\{0,1,\cdots,K-1\}^d$, we get
	\begin{equation*}
		\begin{split}
			|f(\bmx)-\phi_2\circ \calL \circ \phi_1(\bmx)|
			&= |f(\bmx)-f(\bmx_\bmbeta)|+|f(\bmx_\bmbeta)-\phi_2\circ \calL \circ \phi_1(\bmx)|\\
			&\le \omega_f(\tfrac{\sqrt{d}}{K})+|f(\bmx_\bmbeta)-\sum_{\ell=1}^{n}2^{-\ell}\xi_{\bmbeta,\ell}|\\
			&\le \omega_f(\tfrac{\sqrt{d}}{K})+|f(\bmx_\bmbeta)-\bin 0.\xi_{\bmbeta,1}\cdots\xi_{\bmbeta,n}|\le \omega_f(\sqrt{d}\, 2^{-n})+2^{-n}.
		\end{split}
	\end{equation*}
	That is, 
	\begin{equation*}
		\begin{split}
			|f(\bmx)-\phi_2\circ \calL \circ \phi_1(\bmx)|\le \omega_f(\sqrt{d}\, 2^{-n})+2^{-n}\quad \tn{for any $\bmx\in [0,1]^d\backslash \Omega([0,1]^d,K,\delta)$}.
		\end{split}
	\end{equation*}
	Moreover, the fact $\phi_2(\bmy)\in [0,1]$ for any $\bmy\in\R^n$ implies $\|\phi_2\circ\calL\circ\phi_1\|_{L^\infty(\R^d)}\le 1$.
	So we finish the proof.
\end{proof}

\subsection{Proof of Proposition~\ref{prop:phi4j}}
\label{sec:proof:prop}
Before proving Proposition~\ref{prop:phi4j}, let us introduce a notation to simplify the proof.
We use $\calT_m$  for $m\in \N^+$ to denote a ``sawtooth'' function satisfying the following conditions.
\begin{itemize}
	\item $\calT_m:[0,2m]\to [0,1]$ is linear between any two adjacent integers of $\{0,1,\cdots,2m\}$.
	\item $\calT_m(2j)=0$ for $j=0,1,\cdots,m$ and $\calT_m(2j+1)=1$ for $j=0,1,\cdots,m-1$.
\end{itemize}
\begin{figure}[!htp]
	\centering
	\begin{subfigure}[b]{0.2284\textwidth}
		\centering
		\includegraphics[width=0.95\textwidth]{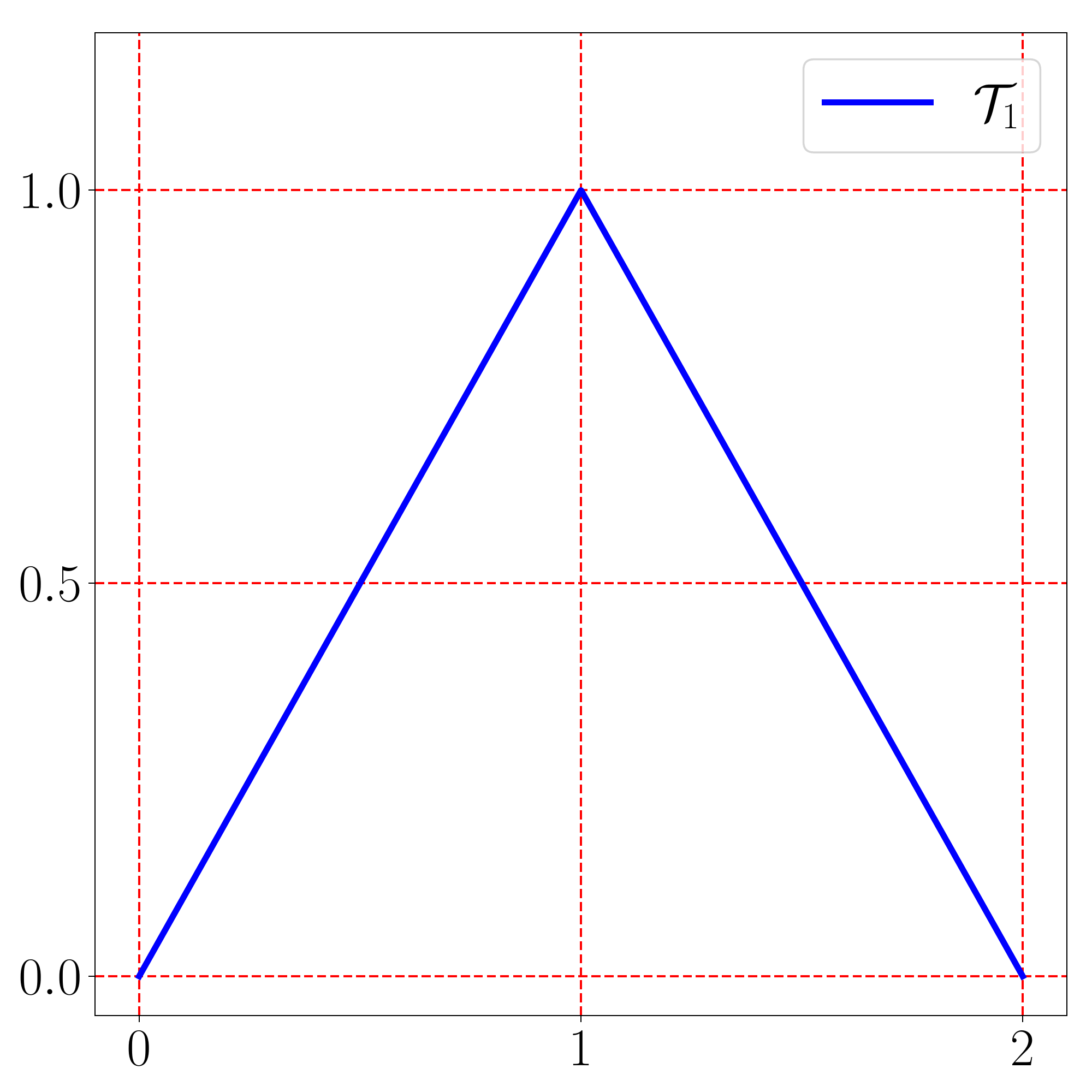}
	\end{subfigure}
	\begin{subfigure}[b]{0.2284\textwidth}
		\centering            \includegraphics[width=0.95\textwidth]{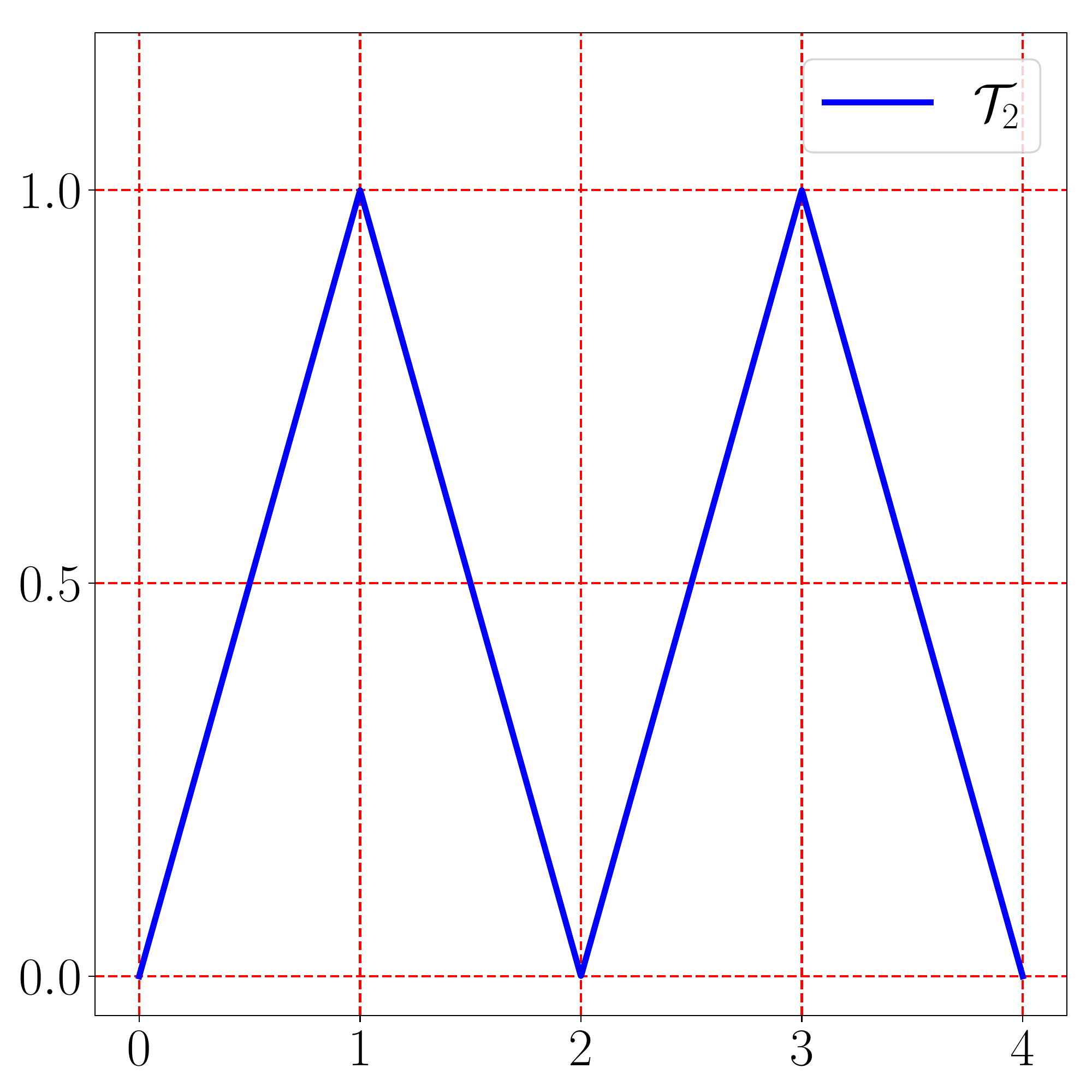}
	\end{subfigure}
	\begin{subfigure}[b]{0.2284\textwidth}
		\centering           \includegraphics[width=0.95\textwidth]{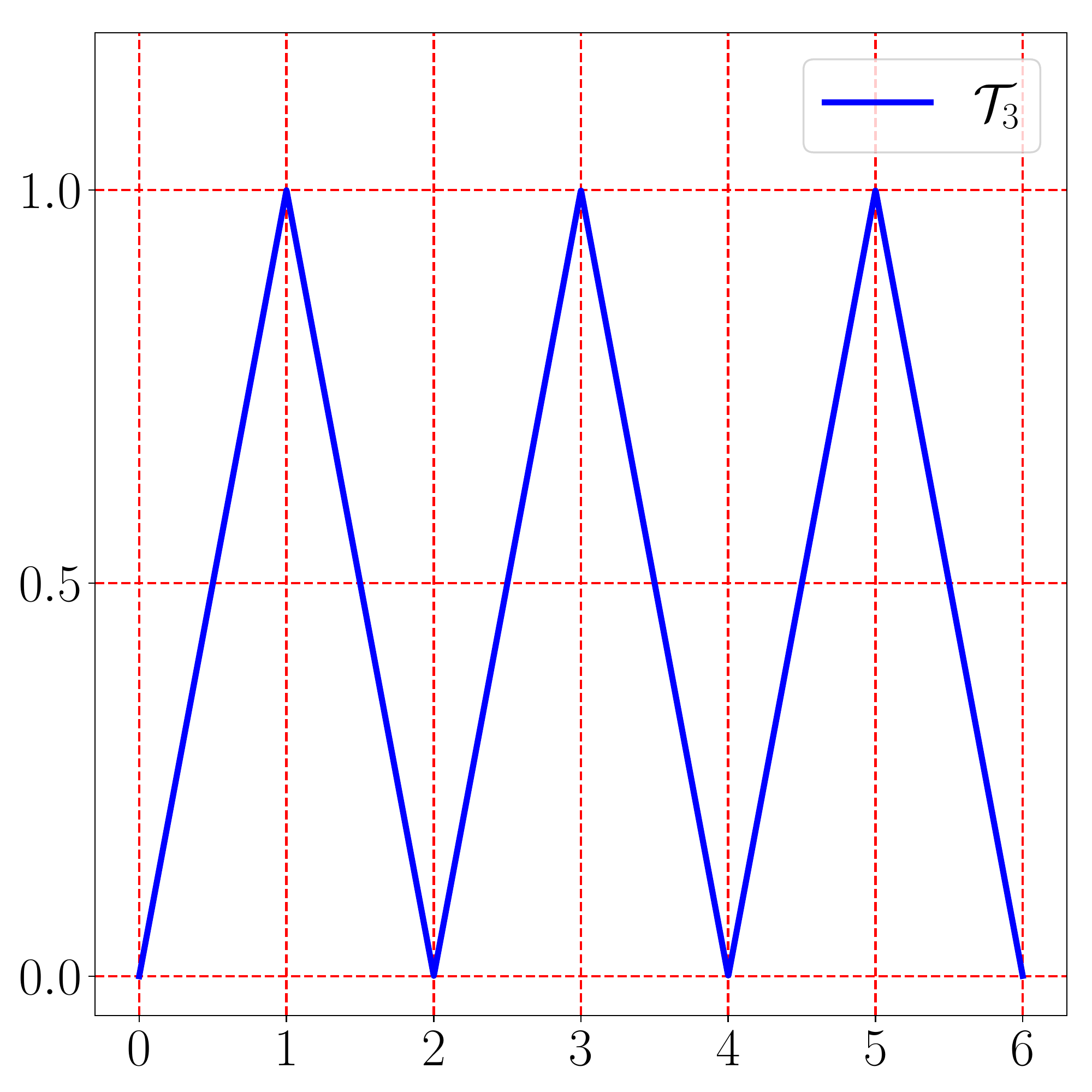}
	\end{subfigure}
	\begin{subfigure}[b]{0.2284\textwidth}
		\centering            \includegraphics[width=0.95\textwidth]{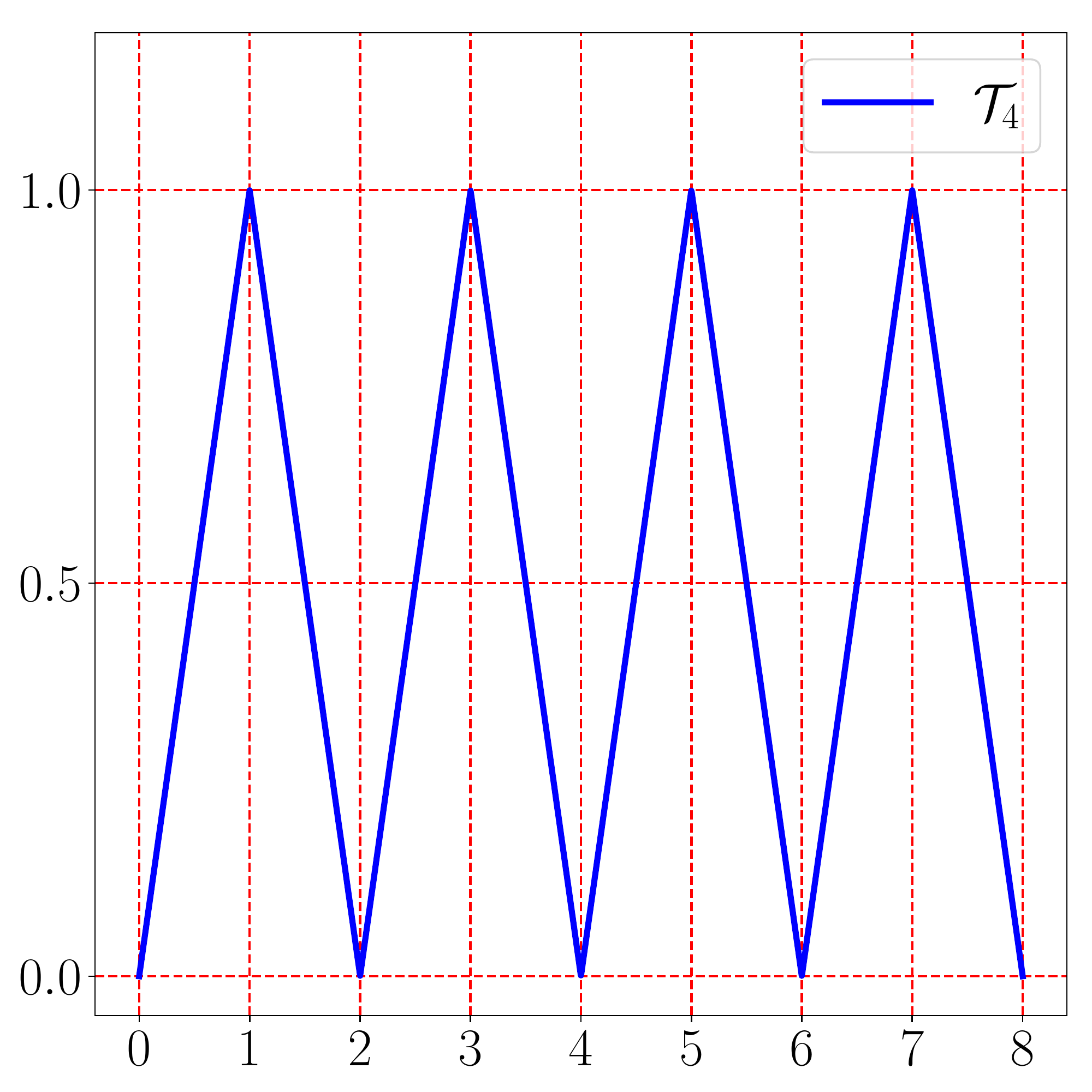}
	\end{subfigure}
	\caption{Illustrations of teeth functions $\calT_1,\ \calT_2,\ \calT_3$ and $\calT_4$.}
	\label{fig:toothFunctions}
\end{figure}

To simplify the proof of Proposition~\ref{prop:phi4j}, we first introduce a lemma based on the ``sawtooth'' function.
\begin{lemma}
	\label{lem:T4J}
	Given any $J\in \N^+$ and $\theta_j\in \{0,1\}$ for $j=1,2,\cdots,J$, set $a=\sum_{j=1}^J \theta_j4^{-j}.$
	Then
	\begin{equation}
		\label{eq:tooth4J}
		\calT_{4^J}(4^{j}a)\in [0,1/3]\quad \tn{if $\theta_j=0$}\quad \tn{and} \quad \calT_{4^J}(4^{j}a)\in [2/3,1]\quad \tn{if $\theta_j=1$},
	\end{equation}
	where $\calT_{4^J}:[0,2\times 4^J]\to [0,1]$ is a ``sawtooth'' function with $4^J$ ``teeth'' defined just above.
\end{lemma}
\begin{proof}
	Fix $j\in \{1,2,\cdots,J\}$, we have
	\begin{equation}\label{eq:3terms}
		\begin{split}
			4^{j}a=4^{j}\sum_{i=1}^J \theta_i4^{-i}
			= \underbrace{\sum_{i=1}^{j-1} \theta_i4^{j-i}}_{\tn{$=4k$ for some $k\in \N$ with $k\le\tfrac{ 4^{J-1}}3$}}
			\quad +\quad  \overbrace{\theta_j}^{\tn{$0$ or $1$}}
			\quad +\quad 
			\underbrace{\sum_{i=j+1}^{J} \theta_i4^{j-i}}_{\tn{$\in [0,\tfrac13)$}}.
		\end{split}
	\end{equation}
	Clearly, 
	\begin{equation*}
		\sum_{i=1}^{j-1} \theta_i4^{j-i}\in \{4k:k\in\N,\,k\le 4^{J-1}/3 \}\quad 
		\tn{and}\quad 0\le \sum_{i=j+1}^{J} \theta_i4^{j-i}\le \sum_{i=j+1}^{J} 4^{j-i}\le 1/3.
	\end{equation*}
	
	If $\theta_j=0$, then Equation~\eqref{eq:3terms} implies
	\begin{equation*}
		4^ja\in [4k,4k+1/3]\quad \tn{for some $k\in \N$ with $k\le 4^{J-1}/3\le 4^J-1$,}
	\end{equation*}
	which implies $\calT_{4^J}(4^ja)\in [0,1/3]$. 
	
	Similarly, if $\theta_j=1$, then Equation~\eqref{eq:3terms} implies
	\begin{equation*}
		4^ja\in [4k+1,4k+1+1/3]\quad \tn{for some $k\in \N$ with $k\le 4^{J-1}/3\le 4^J-1$,}
	\end{equation*}
	which implies $\calT_{4^J}(4^ja)\in [2/3,1]$. 
	So we finish the proof.
\end{proof}
It is worth mentioning that the ``sawtooth'' function $\calT_m$ can be replaced by other functions that also have a key property ``the function values near even integers are much larger than the ones near odd integers''.

With Lemma~\ref{lem:T4J} in hand, we are ready to prove Proposition~\ref{prop:phi4j}.
\begin{proof}[Proof of Proposition~\ref{prop:phi4j}]
	By Lemma~\ref{lem:T4J}, we have
	\begin{equation*}
		\calT_{4^J}(4^{j}a)\in [0,1/3]\quad \tn{if $\theta_j=0$}\quad \tn{and} \quad \calT_{4^J}(4^{j}a)\in [2/3,1]\quad \tn{if $\theta_j=1$}.
	\end{equation*}
	Define $g(x)\coloneqq 3\sigma(x-1/3)-3\sigma(x-2/3)$ for any $x\in\R$, where $\sigma$ is ReLU, i.e., $\sigma(x)=\max\{0,x\}$. 
	See an illustration of $g$ in Figure~\ref{fig:g:and:h}.
	Clearly, 
	\begin{equation*}
		g(x)=0\tn{ if } x\le 1/3\quad  and \quad g(x)=1 \tn{ if } x\ge 2/3.
	\end{equation*}
	Hence, $\phi\coloneqq g\circ \calT_{4^J}$ is the desired function.  Obviously,
	$\phi(x)\in [0,1]$ for any $x\in \R$. To verify  Equation~\eqref{eq:phi4j}, we fix $j\in \{1,2,\cdots,J\}$. If $\theta_j=0$, then $\calT_{4^J}(4^{j}a)\in [0,1/3]$, implying $\phi(4^{j}a)=g\circ \calT_{4^J}(4^j a)=0=\theta_j$.
	If $\theta_j=1$, then $\calT_{4^J}(4^{j}a)\in [2/3,1]$, implying $\phi(4^{j}a)=g\circ \calT_{4^J}(4^j a)=1=\theta_j$.

	It remains to show  $\phi=g\circ \calT_{4^J}$ can be realized by a ReLU network with the expected width and depth. Clearly, $\calT_{4^J}$ is a continuous piecewise linear function, which means it can be implemented by a one-hidden-layer ReLU network. To make our construction more efficient, we introduce another method to implement $\phi$.
	
	\begin{figure}[htp!]
		\centering
		\begin{subfigure}[b]{0.364\textwidth}
			\centering            \includegraphics[width=0.92\textwidth]{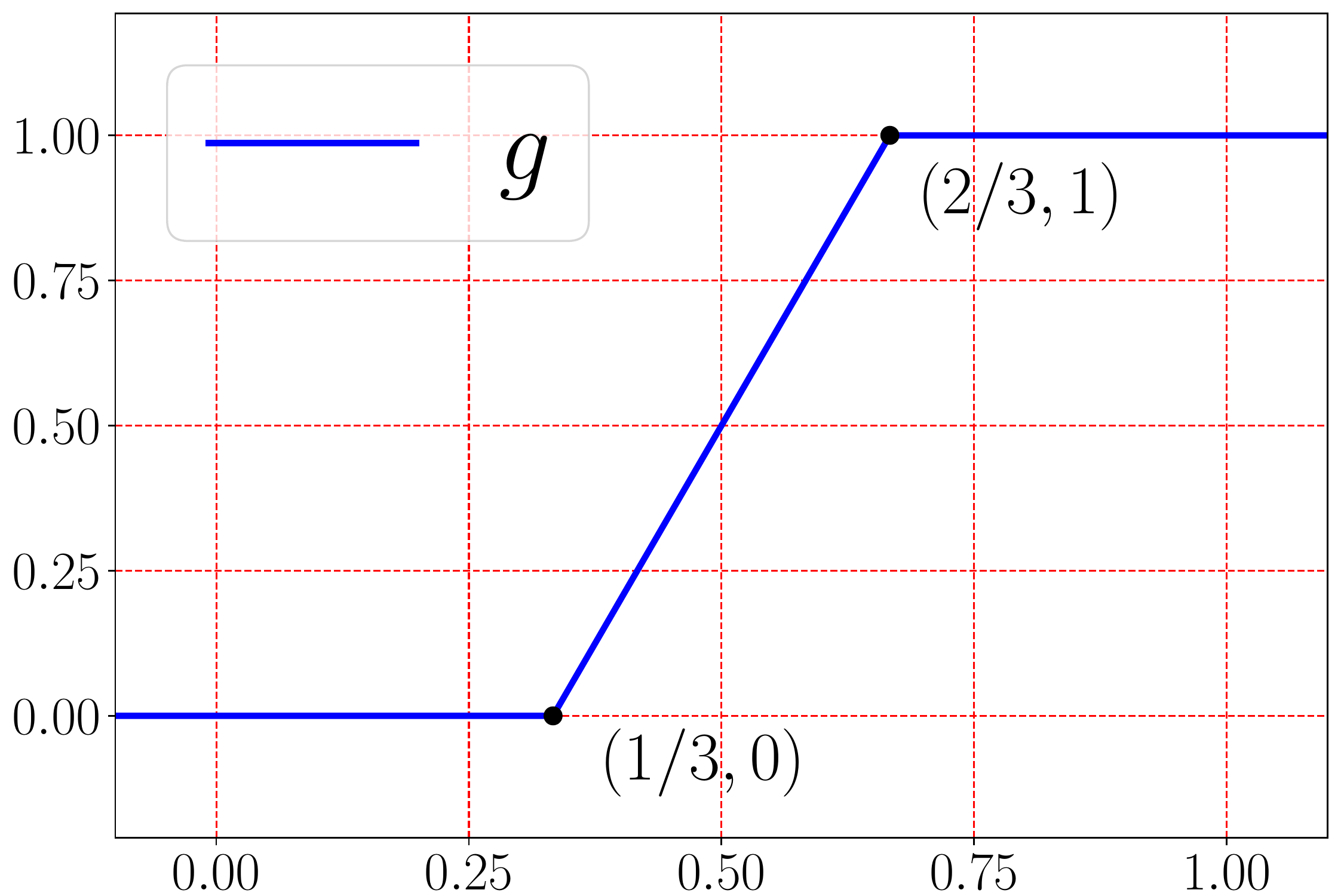}
		\end{subfigure}
		\hspace{8pt}
		\begin{subfigure}[b]{0.364\textwidth}
			\centering           \includegraphics[width=0.92\textwidth]{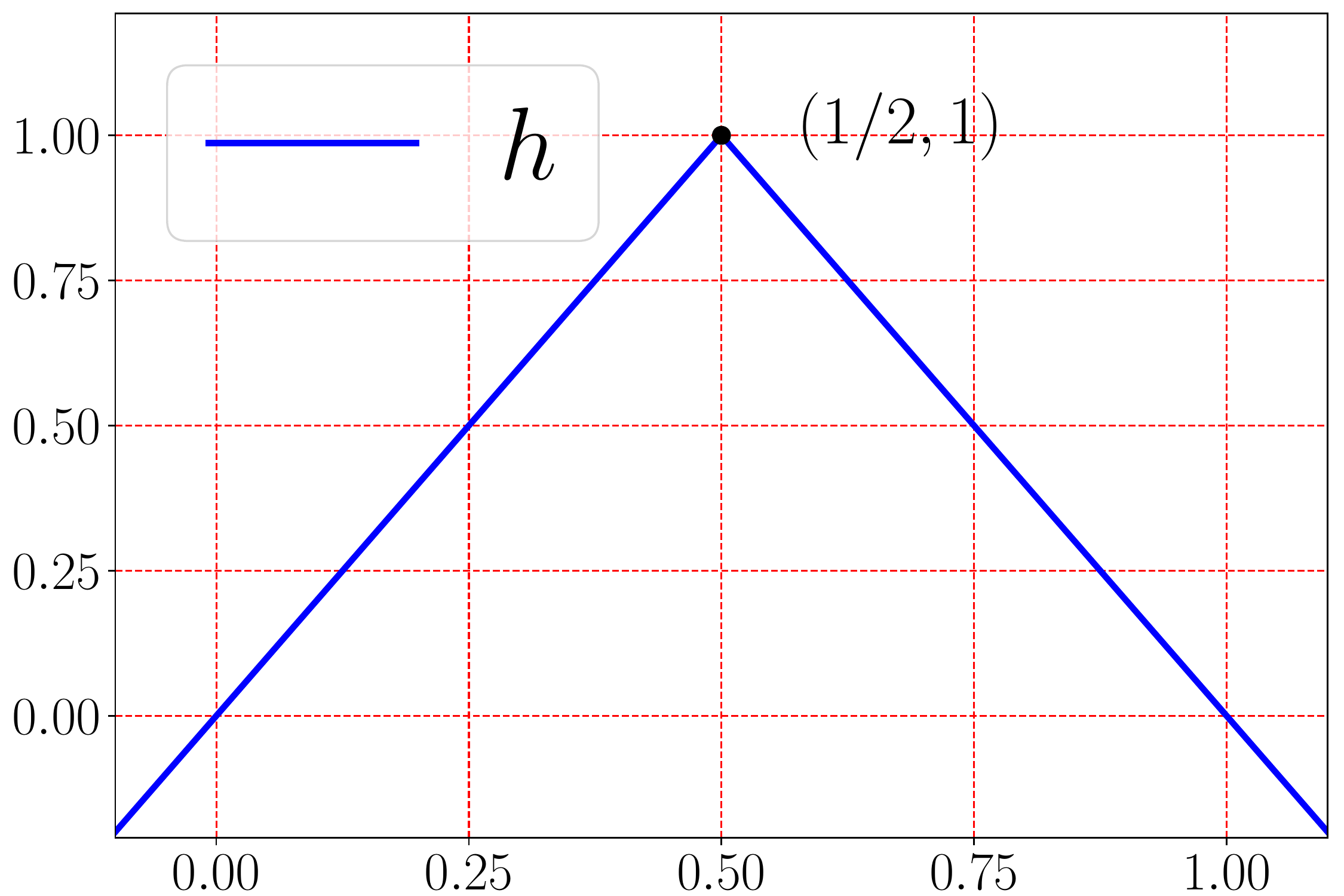}
		\end{subfigure}
		\caption{Illustrations of $g$ and $h$ on $[0,1]$.}
		\label{fig:g:and:h}
	\end{figure}
	
	Define $h(x)\coloneqq 1-2\sigma(x-1/2)-2\sigma(1/2-x)$ for any $x\in [0,1]$. See an illustration of $h$ in Figure~\ref{fig:g:and:h}. Then, it is easy to verify that $\calT_1(2x)=h(x)$ for any $x\in [0,1]$ and $h$ can be implemented by a one-hidden-layer ReLU network with width $2$.  For any $J\in\N^+$, it is easy to verify that
	\begin{equation*}
		\calT_{4^J}(2^{2J+1}x)=\underbrace{h\circ h\circ \cdots \circ h}_{2J+1\ \tn{times}}(x), \quad \tn{for any $x\in [0,1]$.}
	\end{equation*}
	So, $\calT_{4^J}$ limited on $[0,2\times 4^J]=[0,2^{2J+1}]$ can be realized by a ReLU network with width $2$ and depth $2J+1$. It follows that $\phi=g\circ \calT_{4^J}$ can be implemented by a ReLU network with width $2$ and depth $2J+2$, which means we finish the proof.
\end{proof}

\end{document}

%% file: preamble.tex
\usepackage{mathrsfs,amsmath,amsfonts,amssymb,comment}
\usepackage{mathtools,xcolor,enumerate}
\usepackage{dsfont}
\usepackage{bm}
\usepackage{indentfirst}
\usepackage{enumerate} 
\usepackage{tabularx,multirow,array,booktabs}
\usepackage{colortbl}
\usepackage{graphicx}
\graphicspath{{figures/}}%
\usepackage{float}
\usepackage{subcaption} 
\let\tn\textnormal 
\newcommand{\R}{\ifmmode\mathbb{R}\else$\mathbb{R}$\fi}
\newcommand{\N}{\ifmmode\mathbb{N}\else$\mathbb{N}$\fi}
\newcommand{\Z}{\ifmmode\mathbb{Z}\else$\mathbb{Z}$\fi}
\newcommand{\Q}{\ifmmode\mathbb{Q}\else$\mathbb{Q}$\fi}
\newcommand{\bmx}{{\bm{x}}}
\newcommand{\bma}{{\bm{a}}}

\newcommand{\bme}{{\bm{e}}}
\newcommand{\bmy}{{\bm{y}}}

\newcommand{\bmA}{{\bm{A}}}

\newcommand{\bmzero}{{\bm{0}}}
\newcommand{\bmpsi}{{\bm{\psi}}}
\newcommand{\bmphi}{{\bm{\phi}}}

\newcommand{\bmbeta}{{\bm{\beta}}}

\newcommand{\calX}{{\mathcal{X}}}
\newcommand{\calL}{{\mathcal{L}}}
\newcommand{\calO}{{\mathcal{O}}}
\newcommand{\calH}{{\mathcal{H}}}
\newcommand{\calD}{{\mathcal{D}}}
\newcommand{\calR}{{\mathcal{R}}}

\newcommand{\calT}{{\mathcal{T}}}
\newcommand{\calS}{{\mathcal{S}}}
\newcommand{\calN}{{\mathcal{N}}}
\newcommand{\tildephi}{{\widetilde{\phi}}}

\newcommand{\tildef}{{\widetilde{f}}}

\newcommand{\hatbmf}{{\widehat{\bm{f}}}}

\newcommand{\scrF}{{\mathscr{F}}}
\newcommand{\scrH}{{\mathscr{H}}}

\def\middleValue{\tn{mid}}

\DeclareMathOperator*{\argmin}{arg\,min}
\newcommand{\bin}{\tn{bin}\hspace{1.2pt}}
\newcommand{\mystep}[2]{\par \vspace{0.25cm}\noindent\textbf{\hspace{8pt}Step }$#1\colon$ #2 \vspace{0.18cm} \par }
\usepackage{tikz}
\newcommand{\myto}[2][1]{\mathop{
		\vcenter{\hbox{\scalebox{1}[#1]{\tikz{\draw[->,line width=0.72pt] (0,0.5) to (0.69*#2,0.5);}}}}
}}
\let\mycdots\cdots
\def\cdots{{\mycdots}}